\documentclass[11pt]{article}

\usepackage[margin=1in]{geometry}

\usepackage{lipsum}

\usepackage[utf8]{inputenc} 
\usepackage[T1]{fontenc}    
\usepackage{tgtermes}

\usepackage{url}            
\usepackage{booktabs}       
\usepackage{amsfonts}       
\usepackage{nicefrac}       
\usepackage{microtype}      

\usepackage[english]{babel}

\usepackage{amsmath}
\usepackage{enumitem}
\usepackage{amssymb}
\usepackage{graphicx}
\usepackage{bm}
\usepackage{theorem}
\usepackage[colorlinks=true, allcolors=blue]{hyperref}

\usepackage{mathtools}
\newenvironment{proof}{\paragraph{\it Proof.}}{\hfill$\square$}

\usepackage[title]{appendix}
\usepackage{comment}
\usepackage{amsfonts}
\usepackage{dsfont}
\usepackage{hyperref}
\usepackage{lipsum}
\usepackage{dsfont,subcaption,float}
\usepackage{algorithm}

\usepackage{algorithmicx}

\newcommand{\algcomment}[1]{\textcolor{blue!70!black}{\footnotesize{//\hspace{2pt}#1}}}

\usepackage{thm-restate}
\usepackage[size=small,labelfont=bf]{caption}

\usepackage{xcolor}

\theoremstyle{plain}
\newtheorem{theorem}{Theorem}[section]

\newtheorem{proposition}[theorem]{Proposition}
\newtheorem{lemma}[theorem]{Lemma}
\newtheorem{corollary}[theorem]{Corollary}
\theoremstyle{definition}
\newtheorem{definition}[theorem]{Definition}
\newtheorem{assumption}[theorem]{Assumption}
\theoremstyle{remark}

\newtheorem{remark}[theorem]{Remark}

\newcommand{\tssum}{\textstyle \sum}

\DeclarePairedDelimiter\norm{\lVert}{\rVert}

\newcommand{\poly}{\mathrm{poly}}

\newcommand{\mA}{\mathcal{A}}	
\newcommand{\mB}{\mathcal{B}}	

\newcommand{\mR}{\mathcal{R}}	
\newcommand{\mZ}{\mathcal{Z}}	
\newcommand{\indic}[1]{\mathds{1}\left\{ #1 \right\}} 

\newcommand{\PP}{\mathds{P}}    

\newcommand{\Eps}{\mathcal{E}}

\newcommand{\Exs}{\mathbb{E}}

\allowdisplaybreaks

\newif\ifdraft
\drafttrue  

\title{\bf{\LARGE{Tractable Optimality in Episodic Latent MABs}}}

\usepackage{authblk}
\author[1]{Jeongyeol Kwon}
\author[2]{Yonathan Efroni} 
\author[3]{Constantine Caramanis}
\author[4]{Shie Mannor}
\affil[1]{Wisconsin Institute for Discovery, UW-Madison}
\affil[2]{Meta, New York}
\affil[3]{Department of Electrical and Computer Engineering, University of Texas at Austin}
\affil[4]{Department of Electrical Engineering, Technion / NVIDIA}

\begin{document}
\maketitle

\begin{abstract}
We consider a multi-armed bandit problem with $M$ latent contexts, where an agent interacts with the environment for an episode of $H$ time steps. Depending on the length of the episode, the learner may not be able to estimate accurately the latent context. The resulting partial observation of the environment makes the learning task significantly more challenging. 
Without any additional structural assumptions, existing techniques to tackle partially observed settings imply the decision maker can learn a near-optimal policy with $O(A)^H$ episodes, but do not promise more. 
In this work, we show that learning with {\em polynomial} samples in $A$ is possible. We achieve this by using techniques from experiment design. Then, through a method-of-moments approach, we design a procedure that provably learns a near-optimal policy with $O(\poly(A) + \poly(M,H)^{\min(M,H)})$ interactions. In practice, we show that we can formulate the moment-matching via maximum likelihood estimation. In our experiments, this significantly outperforms the worst-case guarantees, as well as existing practical methods.
\end{abstract}

\section{Introduction}
In Multi-Armed Bandits (MABs), an agent learns to act optimally by interacting with an unknown environment. 
In many applications, interaction sessions are short, relative to the complexity of the overall task. As a motivating example, we consider a large content website that interacts with users arriving to the site, by making sequential recommendations. In each such {\em episode}, the system has only a few chances to recommend items before the user leaves the website -- an interaction time typically much smaller than the number of actions or the types of users. In such settings, agents often have access to short horizon episodes, and have to learn how to process observations from different episodes to learn the best adaptation strategy.

Motivated by such examples, we consider the problem of learning a near-optimal policy in Latent Multi-Armed Bandits (LMABs). In each episode with time-horizon $H$, the agent interacts with one of $M$ possible MAB environments (e.g., type of user) randomly chosen by nature. Without knowing the identity of the environment (we call this the {\it latent context}), an agent aims to maximize the expected cumulative reward per episode (see Definition \ref{definition:lmab} for a formal description). The LMAB framework is different from the setting considered in \cite{maillard2014latent} where multiple episodes proceed in parallel without limiting the horizon of an episode $H$. For long horizons $H$, we show that it is possible to find a near-optimal policy and to determine near-optimal actions for each episode, as if we knew the hidden context. If $H \ll A$ where $A$ is the total number of actions, however, this is no longer possible. Instead, we aim to learn the best {\it history-dependent} policy for $H$ time steps.

\subsection{Our Results and Contributions}
\begin{figure}[t]
    \centering
    \includegraphics[width=0.6\textwidth]{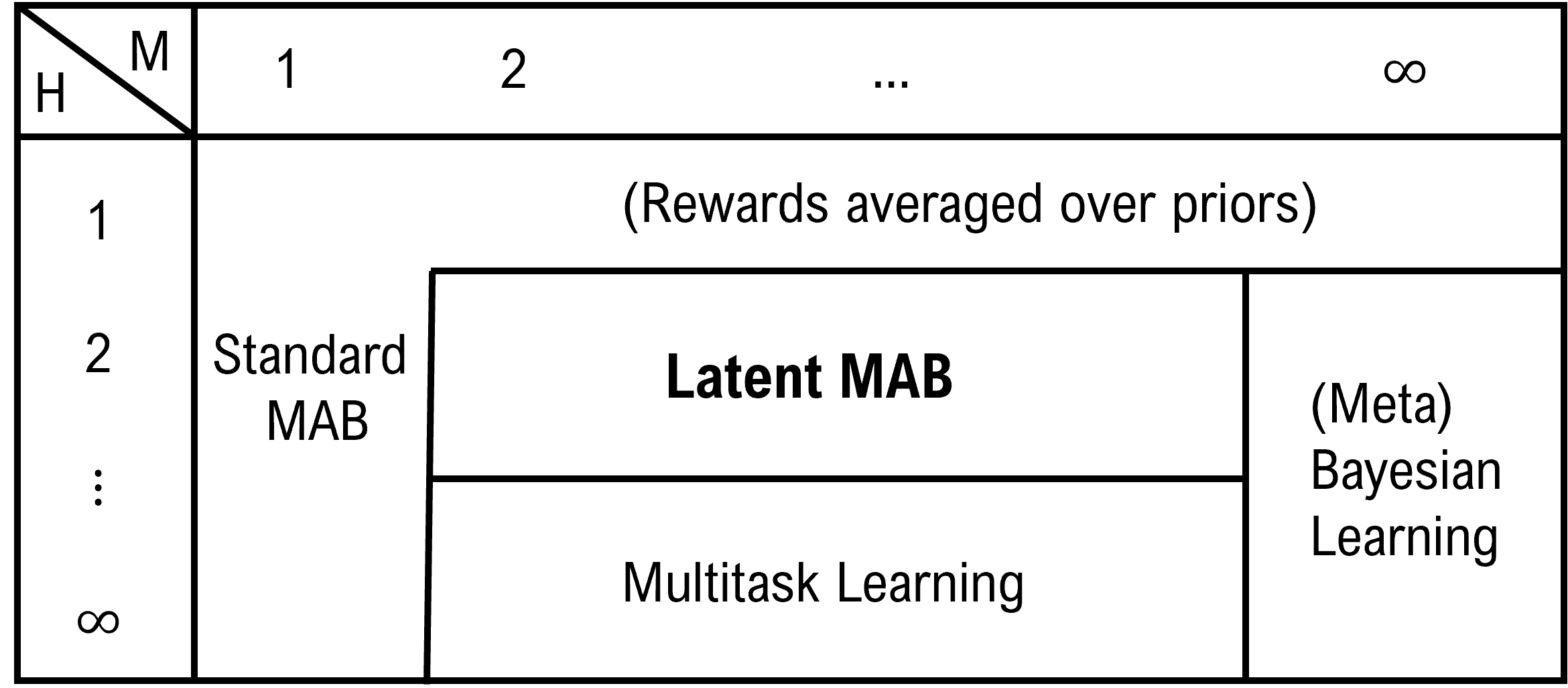}
    \caption{Nature of the problem depends on the length of episodes ($H$) and number of different contexts ($M$)}
    \label{fig:related_work_table}
\end{figure}

In this work, we study the problem of learning a near optimal policy of an LMABs for both long and short horizon $H$. In the long horizon case, we show the problem's latent parameters are learnable. However, for short horizon, the latent parameters may not be identifiable and a different perspective is required.

A naive approach to learn a near optimal policy of an LMAB starts by casting the problem as a Markov Decision Process (MDP). Assume that the reward values are discrete and defined over a finite alphabet of cardinality~$Z$. By defining the state space  as the set of all sequences of history observations an LMAB can be formulated as an MDP with $O((AZ)^H)$ states. Appealing to standard reinforcement learning algorithms ({\it e.g.,} \cite{jaksch2010near, azar2017minimax}) we can learn an $\epsilon$-optimal policy with $O(AZ)^{H} / \epsilon^2$ samples.

A natural way to improve upon the naive approach is through using the unique structure of the LMAB setting. In this work we focus on the following question:
{\it can we learn a near-optimal policy with fewer than $O(AZ)^{H} / \epsilon^2$ samples as the naive approach?} 

Our work answers the above question affirmatively. Specifically, our main contributions are as follows. We show that the dependence of our algorithm is $\poly(A) + \poly(H,M)^{\min(H,M)}$, and thus tractable when either $H$ or $M$ is small, even for very large $A$. That is, we are particularly focused on the setting with a few contexts or relatively short episodes (in comparison to $A$), {\it i.e.,} $M = O(1)$ or $H = O(1)$, where a natural objective is to learn a near-optimal history-dependent policy for $H$ time steps (see also Figure \ref{fig:related_work_table} in Appendix \ref{appendix:additional_related_work}).

\subsection{Related Work} 
\label{subsec:relatedwork}

\paragraph{Comparison of Regimes} Nature of learning a near-optimal policy in LMABs changes depending on the length of episodes $H$ and the number of contexts $M$ (Figure \ref{fig:related_work_table}). For instance, when $H = 1$ or $M = 1$, the problem of learning in LMAB is essentially equivalent to the classical Multi-Armed Bandits problem that has been extensively studied in literature ({\it e.g.,} see \cite{lattimore2020bandit} and references therein). In another well-studied literature of Bayesian learning, each episode is assumed to have a structured prior over contexts ($M \rightarrow \infty$, {\it e.g.,} all expected rewards of arms are independently sampled from beta conjugate priors). Depending on the length of time-horizon $H$, the problem can be solved with Thompson Sampling (TS) \cite{thompson1933likelihood} $H \rightarrow \infty$, or the problem is reduced to learning the prior \cite{kveton2021meta}. When the time-horizon is sufficiently long but finite {\it e.g.,} if $H = \Omega(A / \gamma^2)$ for some separation parameters between a finite number of contexts $M < \infty$, then it is possible cluster observations from every episode. This setting has been studied in the literature of multitask RL which we describe below. In this work, we are particularly focused on the setting with a few contexts and relatively short episodes (in comparison to $A$), {\it i.e.,} $M = O(1)$ and $H = O(1)$, where the most natural objective is to learn a near-optimal history-dependent policy.

\paragraph{Learning priors in multi-armed bandit problems} Several recent works have considered the Bayesian learning framework with short time-horizon \cite{liu2016prior, simchowitz2021bayesian, kveton2021meta}. The focus in this line of work is on the design  of algorithms that learn the prior, while acting with a fixed Bayesian algorithm (e.g., Thompson sampling). While Bayesian learning with short time-horizon may be viewed as a special case of LMABs, the baseline policy we compare ourselves to is the optimal $H$-step policy, which is a  harder baseline than considering a fixed Bayesian algorithm.

\paragraph{Latent MDPs} Some prior work considers the framework of Latent MDPs (LMDPs), which is the MDP generalization of LMAB \cite{hallak2015contextual, steimle2018multi, kwon2021rl, kwon2021reinforcement}. In particular, \cite{kwon2021rl} has shown the information-theoretic limit of learning in LMDPs with no assumptions on MDPs, {\it i.e.,} an exponential number of sample episodes $\Omega(A^M)$ is necessary to learn a near-optimal policy in LMDPs. In contrast, we show that in LMABs, the required number of episodes can be polynomial in $A$. This does not contradict the result in \cite{kwon2021rl}, since their lower bound construction comes from the challenge in state-exploration with latent contexts. In contrast, there is no state-exploration issue for bandits, which enables the polynomial sample complexity in $A$. Furthermore, our upper bound does not require any assumptions or additional information such as good initialization or separations as in \cite{kwon2021rl}. To the best of our knowledge, no existing results are known for learning a near optimal policy of LMAB instances for $M \ge 3$ without further assumptions. 

\paragraph{Multitask RL with explicit clustering} We can cluster observations from each episode if we are given sufficiently long time horizons $H = \tilde{\Omega}(A / \gamma^2)$ in each episode \cite{brunskill2013sample, gentile2014online, hallak2015contextual}. Here, $\gamma > 0$ is the amount of `separation' between contexts such that for all $m \neq m'$, $\max_{a \in \mA} \|\mu_m(a, \cdot) - \mu_{m'}(a, \cdot)\|_{2} \ge \gamma$ where $\mu_m$ is a mean-reward vector of actions in the $m^{th}$ context. We focus on significantly more general cases where there is no obvious way of clustering observations, {\it e.g.,} when $H \ll A$ or $\mu_m$ can be arbitrarily close to some other $\mu_{m'}$ for $m \neq m'$. If we are given a similar separation condition, we also show that a polynomial sample complexity is achievable as long as $H = \tilde{O}(M^2/ \gamma^2)$  (see also section~\ref{subsec:poly_upper_bound_sep}). Note that this could still be in $H \ll A$ regime with large number of actions $A$. 

\paragraph{Learning in POMDPs with full-rank observations} One popular learning approach in partially observable systems is the tensor-decomposition method, which extracts the realization of model parameters from third-order tensors \cite{anandkumar2014tensor, azizzadenesheli2016reinforcement, gopalan2016low}. However, the recovery of model parameters require specific geometric assumptions on the full-rankness of a certain test-observation matrix. Furthermore, most prior work requires a uniform reachability assumption, {\it i.e.,} all latent spaces should be reached with non-negligible probabilities by any exploration policy for the parameter recovery. Recent results in \cite{jin2020sample, liu2022partially} have shown that the uniform reachability assumption can be dropped with the optimism principle. However, they still require the full-rankness of a test-observation matrix to keep the volume of a confidence set explicitly bounded. Since LMAB instances do not necessarily satisfy the full-rankness assumption, their results do not imply an upper bound for learning LMABs.

\paragraph{Reward-Mixing MDPs} Another closely related work is to learn an LMDP (MDP extension of LMAB) with common state transition probabilities, and thus only the reward function changes depending on latent contexts \cite{kwon2021reinforcement}. The authors in \cite{kwon2021reinforcement} have developed an efficient moment-matching based algorithm to learn a near-optimal policy without any assumptions on separations or system dynamics. However, it can only handle the $M = 2$ case with balanced mixing weights $w_1 = w_2 = 1/2$. It is currently not obvious how to extend their result to $M \ge 3$ cases without incurring $O(A^M)$ sample complexity in LMABs.

\paragraph{Regime switching bandits} LMAB may be also seen as a special type of adversarial or non-stationary bandits ({\it e.g.,} \cite{auer2002nonstochastic, audibert2009minimax, garivier2011upper, slivkins2008adapting}) with time steps being specified for when the underlying reward distributions ({\it i.e.,} latent contexts) may change. The standard objective in non-stationary bandits is to find the best stationary policy in hindsight with unlimited possible contexts. Recently, \cite{zhou2021regime} considered a non-stationary bandit with a finite number of contexts $M = O(1)$ and the objective of finding the optimal {\it history-dependent} policy. Their setting and goal subsume our goal of learning the optimal policy in LMABs; however, results in \cite{zhou2021regime} require linear independence between reward probability vectors, and thus their setting essentially falls into the category of tractable POMDPs with full-rank observations. 

\paragraph{Miscellaneous} There are other modeling approaches more suited for personalized recommendation systems where multiple episodes proceed in parallel without limits on the time-horizon \cite{maillard2014latent, gentile2014online, chawla2020gossiping, hu2021near, kwon2022coordinated}. In such problems, the goal is to quickly adapt policies for individual episodes assuming certain similarities between tasks. In contrast, in episodic LMAB settings, we assume that every episode starts in a sequential order, and our goal is to learn an optimal history-dependent policy that can maximize rewards for a single episode with limited time horizon.

\section{Preliminaries}
\subsection{Problem Setup}
We define the problem of {\bf episodic} latent multi-armed bandits with time-horizon $H \ge 2$ as follows: 
\begin{definition}[Latent Multi-Armed Bandit (LMAB)]
    \label{definition:lmab}
    LMAB is a tuple $\mB = ( \mathcal{A}, \{w_m\}_{m=1}^M, \{\mu_m\}_{m=1}^M )$, where $\mathcal{A}$ is a set of actions, $\{w_m\}_{m=1}^M$ are the mixing weights such that a latent context $m$ is randomly chosen with probability $w_m$, and $\mu_m$ is the \emph{model parameter} that describes a reward distribution, {\it i.e.,} $\PP_{\mu_m} (r \mid a) := \PP(r \mid m, a)$, according to an action $a \in \mA$ conditioning on a latent context $m$. (NB: each $\mu_m$ represents a probability model, not necessarily a mean reward vector). 
\end{definition}
We do not assume a priori knowledge of mixing weights. The bulk of this paper considers discrete reward realizations, when the support of the reward distribution is finite and bounded. In Appendix \ref{appendix:Gaussian_Rewards}, we also show that our results can be adapted to Gaussian reward distributions.
\begin{assumption}[Discrete Rewards]
\label{assumption:reward_dist}
    The reward distribution has finite and bounded support. The reward attains a value in the set $\mathcal{Z}$. We assume that for all $z\in \mathcal{Z}$ we have $|z| \le 1$. We denote the cardinality of $ \mathcal{Z}$ as $Z$ and assume that $Z = O(1)$. 
\end{assumption}
As an example, Bernoulli distribution satisfies Assumption \ref{assumption:reward_dist} with $\mZ=\{0,1\}$ and $Z = 2$.  We denote the probability of observing a reward value $z$ by playing an action $a$ as $\mu_m(a,z) \coloneqq \PP(r = z \mid m, a)$ in a context $m$. 
We often use $\mu_m$ as a reward-probability vector in $\mathbb{R}^{AZ}$ indexed by a tuple $(a,z) \in \mA \times \mZ$.

At the beginning of every episode, a latent context $m \in [M]$ is sampled from a mixing distribution $\{w_m\}_{m=1}^M$ and fixed for $H$ time steps, however we cannot observe $m$ directly. We consider a policy class $\Pi$ which contains all history-dependent policies $\pi: (\mA \times \mZ)^* \rightarrow \mA$. 
Our goal is to find a near optimal policy $\pi \in \Pi$ that maximizes the expected cumulative reward $V$ for each episode 
$V^\star = \max_{\pi \in \Pi} V(\pi) := \Exs^\pi \left[ \sum_{t=1}^H r_t  \right],$
where the expectation is taken over latent contexts and rewards generated by an LMAB instance, and actions following a policy $\pi$. 

\begin{definition}[Approximate Planning Oracle]
    \label{definition:planning_oracle}
    A planning oracle receives an LMAB instance $\mB$ and returns an $\epsilon$-approximate policy $\pi$ such that $V^\star - V(\pi) \le \epsilon$.
\end{definition}
Concretely, the point-based value-iteration (PBVI) algorithm \cite{pineau2006anytime} is an $\epsilon$-approximate planning algorithm which runs in time $O(HMAZ (H^2/\epsilon)^{O(M)})$. 

\subsection{Experimental Design}
\label{subsec:expdesign}

We now give a high-level overview on experimental design techniques used in this work. Suppose we are given a matrix $\Phi \in \mathbb{R}^{d \times k}$ where $d \gg k$. Define a distribution over the rows of $\Phi$, $\rho\in \Delta^d$ an element in the $d$-dimensional simplex. We want to select a small subset of rows of $\Phi$ which minimizes $g(\rho)$ defined below:
\begin{align}
    G(\rho) = \sum_{i \in [d]} \rho(i) \Phi_{i,:} \Phi_{i,:}^\top, \qquad g(\rho) = \max_{i \in [d]} \|\Phi_{i,:}\|_{G(\rho)^{-1}}^2, \label{eq:g_optimal_design}
\end{align}
where $\Phi_{i,:} \in \mathbb{R}^k$ be the $i^{\mathrm{th}}$ row of $\Phi$ and $G(\rho) \in \mathbb{R}^{k\times k}$. To achieve this task we use results from the experimental design literature. The following theorem shows the existence of $\rho$ that minimizes $g(\rho)$ with \emph{a small support} over the row indices of $\Phi$: 
\begin{theorem}[Theorem 4.4 in \cite{lattimore2020learning}]
    \label{theorem:small_core_set}
    There exists a probability distribution $\rho$ such that $g(\rho) \le 2k$ and $|supp(\rho)| \le 4k \log\log k + 16$. Furthermore, we can compute such $\rho$ in time $\tilde{O}(dk^2)$. 
\end{theorem}
As noted in \cite{lattimore2020learning}, Theorem \ref{theorem:small_core_set} can be obtained from results of Chapter 3 in \cite{todd2016minimum}. Using this fundamental theorem, \cite{lattimore2020learning} showed the following proposition:
\begin{proposition}[Proposition 4.5 in \cite{lattimore2020learning}]
    \label{proposition:core_set}
    Let $\rho$ be a distribution over the rows of $\Phi$ that satisfies the condition of Theorem \ref{theorem:small_core_set}. Suppose a vector $\mu \in \mathbb{R}^d$ can be represented as a sum $\mu = v + \Delta$ where $v \in V$, $\|\Delta\|_{\infty} \le \epsilon_0$. Let $\eta$ be any small noise with $\eta \in [-\epsilon_1, \epsilon_1]^d$. Then $\|\Phi \hat{\theta} - \mu\|_\infty \le \epsilon_0 + (\epsilon_0 + \epsilon_1) \sqrt{2k}$ where
    \begin{align}
        \hat{\theta} = G(\rho)^{-1} \sum_{i \in [d]} \rho(i) (\mu(i) + \eta(i)) \Phi_{i,:}. \label{eq:coordinate_recovery}
    \end{align}
\end{proposition}
Crucially, we use Proposition \ref{proposition:core_set} to reduce the sample complexity in $A$. 


\subsection{Wasserstein Distance}
\label{subsec:wasserstein}
We now give a brief overview on the Wasserstein distance and its applications in latent mixture models. Wasserstein distance is a convenient error metric to measure the parameter distance between two latent models $\{(w_m, \nu_m)\}_{m=1}^M$ and $\{(\hat{w}_m, \hat{\nu}_m)\}_{m=1}^M$, where $w_m, \hat{w}_m$ are mixing probabilities and $\nu_m, \hat{\nu}_m$ are some parameters for individual contexts. 

Wasserstein distance is defined as follows. Let $\nu$ be a finite-support distribution over $\{\nu_m\}_{m=1}^M$ with probabilities $\{w_m\}_{m=1}^M$, {\it i.e.,} $\gamma = \sum_{m=1}^M w_m \delta_{\nu_m}$ where $\delta_v$ is a Direc-delta distribution with a single mass on $v \in \mathbb{R}^n$. Similarly with parameters $\{(\hat{w}_m, \hat{\nu}_m)\}_{m=1}^M$, define an atomic distribution $\hat{\gamma} = \sum_{m} \hat{w}_m \delta_{\hat{\nu}_m}$. We define a Wasserstein distance between $\gamma$ and $\hat{\gamma}$ with respect to $l_{\infty}$ norm as the following:
\begin{align*}
    W(\gamma, \hat{\gamma}) := \inf_{\Gamma} \Exs_{(m,m') \sim \Gamma} \left[ \|\nu_m - \hat{\nu}_{m'}\|_{\infty} \right] = \inf_{\Gamma} \sum_{(m, m') \in [M]^2} \Gamma(m, m') \cdot \|\nu_m - \hat{\nu}_{m'}\|_{\infty}, 
\end{align*}
where the infimum is taken over all couplings over joint distributions $\nu$ and $\hat{\nu}$ which are marginally distributed as $\nu$ and $\hat{\nu}$ respectively, {\it i.e.,}
\begin{align}
    \Gamma(m,m') \in \mathbb{R}_+^{M \times M}: \tssum_{m'=1}^M \Gamma(m,m') = w_m, \ \tssum_{m=1}^M \Gamma(m,m') = \hat{w}_m. \label{eq:gamma_set}
\end{align}

One nice property of Wasserstein metric is that the distance measure is invariant to permutation of individual components, and flexible with arbitrarily small mixing probabilities or close parameters for different contexts \cite{wu2020optimal, doss2020optimal}.

\paragraph{Additional notation.} For any quantity $q$ with respect to the true LMAB, $\mB$, we use $\hat{q}$ to refer to its corresponding empirical estimate. 
We use $w_{\mathrm{min}} := \min_{m\in[M]} w_m$ for the minimum mixing weight. For any $l^{th}$-order tensor $T_l \in \mathbb{R}^{n \times ... \times n}$ ($n$ repeated $l$ times), we denote $\|T_l\|_{\infty}$ for the element-wise largest absolute value. For any vector in $v \in \mathbb{R}^n$ in dimension $n \in \mathbb{N}_+$, we use $v^{\bigotimes l}$ to denote a degree $l$ tensorization of $v$. We often use $\|\mu_m(a,\cdot) - \hat{\mu}_m(a,\cdot)\|_1$ to mean the $L_1$ statistical distance between reward distributions: $\sum_{z \in \mZ} |\mu_m(a,z) - \hat{\mu}_m(a,z)|$. Lastly, we denote by $a_t$ and $r_t$ as the action and reward realizations observed at time step $t\in [H]$.


\section{Sample-Complexity of Learning LMABs}
\label{section:main_result}


In this section, we develop our main algorithm that learns a near-optimal policy of an LMAB instance with $O(\poly(A) + \poly(M,H)^{\min(M,H)})$ samples. Towards achieving this goal, we first elaborate on a low rank representation of the LMAB problem. Then, we show how experimental design techniques can assist in utilizing this structure to improve over the naive upper bound to the problem.


\subsection{Dimensionality Reduction via Experimental Design} \label{subsec:exp_des}
When $H=1$, estimating the latent context is not possible, and executing a standard MAB strategy (e.g., UCB) is optimal. 
However, to obtain a near-optimal history-dependent policy with longer time-horizons $H > 1$, tracking the mean  reward of an action is not enough, since rewards from previous time steps are correlated with current and future rewards. Moreover, when $H > 1$ we can collect more information such as correlations between different actions. This allows the use of higher-order statistics.

Specifically, we consider a model-based learning approach from the $l^{\mathrm{th}}$-order moments of reward observations. Since there are $A$ actions, this implies that we have $O(A^l)$ quantities to estimate, which would incur $O(A^l)$ sample complexity if we separately measure every correlation between $l$ pairs of actions. On the other hand, when $M  \ll A$, the distributions $\mu_m$ (recall these are reward probabilities for each action, in each context) occupy only a $M$-dimensional subspace of $AZ$-dimensional space: $\mathbf{U} := \mathrm{span}\{\mu_1, \mu_2, \ldots, \mu_M\}.$ Thus if we could estimate $\mathbf{U}$, and project all observations to $\mathbf{U}$, we could remove the unfavorable dependence on $A$.



Even when the subspace $\mathbf{U}$ is known a priori, dimensionality reduction with bandit feedback is non-trivial. We need to compute directly the projection onto $\mathbf{U}$ of estimates of the reward probability vectors, $\{\mu_m\}_{m=1}^M$. To see the challenge, let $\{\beta_j\}_{j=1}^M \subset \mathbb{R}^{AZ}$ be an orthonormal basis of $\mathbf{U}$. To compute the projected estimate $\mu_m^\top \beta_j$, we need a sampling policy $\pi$ such that $\pi(a) \propto \sum_{z} |\beta_j(a,z)|$. The variance of this sampling policy can be as much as $O(\|\beta_j\|_1^2)$, which in general scales with $\|\beta_j\|_1^2 = O(A)$. Therefore, reliable estimation for any statistics in the reduced dimension would still have a dependence on $A$. In particular, since our approach is based on higher-order method-of-moments, an estimation of $l^{\mathrm{th}}$-order statistics for $l \ge 2$ would require $\Omega(A^{l/2})$ samples.


The general idea of dimensionality reduction is not fatally flawed. Instead, we need to avoid the pitfalls of the approach outlined above. First, the calculation above shows that we need to control the $\|\cdot\|_1$-norm of the vectors $\beta_j$, ideally, $\|\beta_j\|_1 = O(1)$. Second, we do not need to approximate our the reward-probability vectors in $l_2$-norm, but only in a $\|\cdot\|_{\infty,1}$-norm. Indeed, to compute an $\epsilon$-optimal policy, we only need a good estimate for each $\mu_m$ such that $\max_{a \in \mA} \|\mu_m(a, \cdot) - \hat{\mu}_m(a, \cdot)\|_1 \le \epsilon$. 

The key is to show that we can choose $\beta_j$'s to be a subset of the standard basis in $\mathbb{R}^{AZ}$ (hence they will have $\|\beta_j\|_1 = 1$), in such a way that guarantees the approximation quality for $\mu_m$ from estimating $\mu_m^\top \beta_j$'s. In terms of the original problem, the existence of such a subset of the standard basis is equivalent to the existence of a small set of informative action-value pairs, that are sufficiently correlated with all other action-value pairs. This is called a \emph{core set} in the experimental design literature, and its existence in our context, is an important consequence of the Kiefer–Wolfowitz theorem. Specifically, we can select a core set of coordinates with the following crucial lemma:
\begin{lemma}
    \label{lemma:find_sampling_coordinates}
    Let $\mathbf{U}$ be a given $k$-dimensional linear subspace in $\mathbb{R}^d$ where $d \gg k$ and let $u, \bar{u} \in \mathbf{U}$. There exists an algorithm that runs in time $\tilde{O}(dk^2)$ and returns the following.
    
    \begin{enumerate}
        \item A core set of at most $n = 4k \log\log k + 16$ coordinates $\{i_j\}_{j=1}^n \subseteq [d]$ such that          
        $$
        {\norm{u - \bar{u}}_{\infty} \le \sqrt{2k} \max_{j\in [n]} |u(i_j) - \bar{u}(i_j)|.}
        $$
        \item A linear transformation that maps $[\bar{u}(i_1),\ldots,\bar{u}(i_n)]$ to its corresponding $\bar{u}\in \mathbf{U}.$
    \end{enumerate}
\end{lemma}
Note that in our setting, $d = AZ$ and $k = M$. We prove Lemma \ref{lemma:find_sampling_coordinates} in Section \ref{appendix:find_sampling_coordinates}, essentially as a corollary of the Kiefer–Wolfowitz theorem (and its geometric interpretation) for (near)-optimal experimental design \cite{kiefer1960equivalence, todd2016minimum}. 

In the context of bandits and RL, the Kiefer–Wolfowitz theorem has been used to study how the misspecification in linear representation changes the problem landscape ({\it e.g.,} \cite{lattimore2020learning, modi2020sample}). However, experimental design has not been previously used for dimensionality reduction for problems with bandit feedback (as far as we know). We believe it is a powerful tool. 

A direct consequence of Lemma \ref{lemma:find_sampling_coordinates} is an algorithm to find a set of coordinates of size $\tilde{O}(M)$ that are sufficient to reconstruct the latent reward model $\{\mu_m\}_{m=1}^M$ where $\mu_m\in \mathbb{R}^{AZ}$. We refer to this set of coordinates as the set of core action-value pairs. 
\begin{corollary}
    \label{corollary:action_event_sampling}
    Suppose the subspace $\mathbf{U} = \mathrm{span}\{\mu_1, \ldots, \mu_M\}$ is given. Then for any $\hat{\mu} \in \mathbf{U}$, there exists an algorithm that runs in time $\tilde{O} (ZAM^2)$ and returns the following.
    \begin{enumerate}
        \item A set of core action-value pairs $\{(a_j, z_j)\}_{j=1}^n\subseteq \mathcal{A}\times \mathcal{Z}$ of size at most ${n = 4M \log\log M + 16}$, such that for all $m \in [M]$
        $$
        {\max_{a \in \mA} \|\mu_m(a, \cdot) - \hat{\mu}_m (a, \cdot) \|_1 \le 2Z \sqrt{2M} \cdot \max_{j \in [n]} \left| \nu_m (j)- \hat{\nu}_m (j) \right|},
        $$
        where $\nu_m, \hat{\nu}_m \in \mathbb{R}^n$ such that $\nu_m (j) := \mu_m(a_j, z_j)$ and $\hat{\nu}_m (j) := \hat{\mu}_m(a_j, z_j)$.
        \item A linear transformation that maps $[\hat{\nu}(i_1),\ldots,\hat{\nu}(i_n)]$ to its corresponding $\hat{\mu}\in \mathbf{U}.$
    \end{enumerate}
\end{corollary}


Assuming access to the subspace $\mathbf{U}$ (we  remove this assumption in Section~\ref{subsec:general_lmab_final}), Corollary~\ref{corollary:action_event_sampling} implies a strategy for estimating the latent model parameters $\{\mu_m\}_{m=1}^M$: obtain core action-value pairs $\{(a_j,z_j) \}_{j=1}^n$, estimate $\{\nu_m\}_{m=1}^M$, and construct latent model parameters via a linear transformation of $\{\nu_m\}_{m=1}^M.$ In the following sections, we show that by matching higher-order moments we can estimate $\{\nu_m\}_{m=1}^M$ to sufficiently good accuracy with only polynomial dependence in $A$.

\subsection{$H \ge 2M-1$: Identifiable Regime in Wasserstein Metric} \label{subsec:general_long_horizon}
In the regime where that $H \ge 2M-1$ we can measure moments up to order $(2M-1)^{\mathrm{th}}$. Then, we can leverage recent advances in learning parameters of mixture distributions from higher-order moments \cite{wu2020optimal, doss2020optimal}. That is, we can recover $\{\nu_m\}_{m=1}^M$ by estimating  $2M-1$ higher-order moments and find $\{\hat{\nu}_m\}_{m=1}^M$ that matches them. This further emphasizes the importance of the dimensionality reduction we take to obtain the set of core action-values pairs; otherwise, a moment-based approach to recover $\{\mu_m\}_{m=1}^M$ would have an exponential dependence in $A$.


We now elaborate on the estimation procedure of $\{\nu_m\}_{m=1}^M$ from higher-order moments. 
We define the $l$-order tensor as $T_l := \sum_{m=1}^M w_m \nu_m^{\bigotimes l}$. For an LMAB instance, we can access the tensor $T_l$ using observational data by simply estimating correlations between $l$ core action-value pairs within each episode. Specifically, for any $\mathcal{I} = (i_1, i_2, ..., i_l) \in [n]^l$ the tensor $T_l (i_1, i_2, ..., i_l)$ is also given by 
\begin{align}
    T_l (i_1, i_2, ..., i_l) = \Exs^{\pi_{\mathcal{I}}} \left[ \Pi_{t=1}^l \indic{r_t = z_{i_t}} \right], \label{eq:expectation_tensor}
\end{align}
where $\pi_{\mathcal{I}}$ is a policy that performs the sequence of actions $(a_{i_1}, a_{i_2}, ..., a_{i_l})$ for $t = 1, \ldots, l$.


Suppose we find estimators $\{(\hat{w}_m, \hat{\nu}_m)\}_{m=1}^M$ such that moments match $T_l$ up to error $\delta$ for all $l \in [2M-1]$. The results in \cite{wu2020optimal, doss2020optimal} imply that the closeness in moments of up to $(2M-1)^{\mathrm{th}}$ degree implies the closeness in model parameters: 
\begin{lemma}
    \label{lemma:moment_closeness}
    Suppose $\|T_l - \sum_{m=1}^M \hat{w}_m \hat{\nu}_m^{\bigotimes l}\|_{\infty} < \delta$ for all $l = 1, 2, ..., 2M-1$ for some sufficiently small $\delta > 0$. Then,
    \begin{align}
        \inf_{\Gamma} \sum_{(m, m') \in [M]^2} \Gamma(m, m') \cdot \|\nu_m - \hat{\nu}_{m'}\|_{\infty} \le O \left(M^3 n \cdot \delta^{-1/(2M-1)} \right) \label{eq:nu_guarantee},
    \end{align}
    where $\Gamma$ is a joint distribution over $(m, m') \in [M]^2$ satisfying \eqref{eq:gamma_set}.
\end{lemma}
The form of guarantee given for $\{\hat{w}_m, \hat{\nu}_m\}_{m=1}^M$ is in the {\it Wasserstein distance} between two latent model parameters. A useful property of the Wasserstein metric is that the distance measure is invariant to permutation of individual components, and flexible with arbitrarily small mixing probabilities or arbitrarily close components~\cite{wu2020optimal} (see also Appendix \ref{subsec:wasserstein} for the review on Wasserstein distance). 

Once we obtain estimates of $\{\nu_m\}_{m=1}^M$, we can estimate $\{\hat{\mu}_m\}_{m=1}^M$ as implied by Corollary \ref{corollary:action_event_sampling} (see Appendix \ref{appendix:theorem_lmab} for further details). We then show that for any history-dependent policy $\pi$, the expected cumulative rewards of $\pi$ are approximately the same for any close LMAB instances. 
\begin{proposition}
    \label{proposition:total_variation_diff}
    Let $\mB = ( \mA, \{w_m\}_{m=1}^M, \{\mu_m\}_{m=1}^M )$ and $\hat{\mB} = (\mA, \{\hat{w}_m\}_{m=1}^M, \{\hat{\mu}_m\}_{m=1}^M )$ be any two LMABs. Then, for any history-dependent policy $\pi: (\mA \times \mZ)^* \rightarrow \mA$, we have
    \begin{align}
        |V(\pi) - \hat{V}(\pi)| \le H^2 \cdot \inf_{\Gamma} \sum_{(m,m') \in [M]^2} \left( \Gamma(m, m') \cdot \max_{a \in \mA} \|\mu_m(a, \cdot) - \hat{\mu}_{m'} (a, \cdot)\|_{1} \right), \label{eq:value_to_wasserstein}
    \end{align}
    where the infimum over $\Gamma$ is taken over joint distributions over $(m, m')$ satisfying \eqref{eq:gamma_set}. 
\end{proposition}
From Corollary \ref{corollary:action_event_sampling} we have $\max_{a\in\mA}\|\mu_m(a, \cdot) - \hat{\mu}_{m'} (a, \cdot)\|_1 \le 2Z\sqrt{2M} \|\nu_m - \hat{\nu}_{m'}\|_{\infty}$ for any $m, m' \in [M]$. Plugging this into Proposition \ref{proposition:total_variation_diff}, we have 
\begin{align*}
    |V(\pi) - \hat{V}(\pi)| \le \poly(H, Z, M, n) \cdot \delta^{-1/(2M-1)}. 
\end{align*} 
Thus, using Lemma \ref{lemma:moment_closeness} with $\delta < (\poly(H, Z, M, n) / \epsilon)^{2M-1}$, we can conclude that any $\epsilon$-optimal policy for $\hat{\mB}$ is $O \left(\epsilon\right)$-optimal for the underlying LMAB $\mB$.

\subsection{$H < 2M-1$: Unidentifiable Regime with Short Time-Horizon}
\label{subsec:general_short_horizon}
A more interesting regime is the one in which the time-horizon $H$ is smaller than the required degree of moments $2M-1$. For such a setting \emph{we cannot measure moments of degree higher than $H$}. Therefore, if $H<2M - 1$, we cannot rely on the identifiablility of the underlying LMAB model. 

Instead, we make the following claim: to compute an optimal policy only for time horizon $H < 2M-1$, we only need to match the $H^{\mathrm{th}}$ order moments. This is formalized in the next lemma. 
\begin{lemma}
    \label{lemma:short_time_moment_to_values}
    Suppose that $\|\sum_{m=1}^M w_m \mu_m^{\bigotimes H} - \sum_{m=1}^M \hat{w}_m \hat{\mu}_m^{\bigotimes H}\|_{\infty} \le \delta$, then for any history dependent policy $\pi \in \Pi$, we have
    \begin{align*}
        |V(\pi) - \hat{V}(\pi)| \le H Z^H \cdot \delta.
    \end{align*}
\end{lemma}
Therefore, it is sufficient to find estimates of mixing weights and latent models that match the measured $H^{\mathrm{th}}$-order moment, and then compute an optimal policy for the estimated model.  Lemma~\ref{lemma:short_time_moment_to_values} is natural for discrete reward distributions with bounded support. Interestingly, we extend this result to Gaussian rewards, which are continuous and unbounded, in Appendix \ref{appendix:Gaussian_Rewards}. 

Hence, a natural idea is to estimate parameters $\{\hat{w}_m\}_{m=1}^M, \{\hat{\mu}_m\}_{m=1}^M$ that satisfy the condition in Lemma \ref{lemma:short_time_moment_to_values} without incurring $O(A^H)$ sample complexity. This is possible if we have good estimates of moment-matching parameters for a set of core action-values. 
\begin{proposition}
    \label{proposition:short_time_core_set_moments}
    For any given $l \ge 1$, if $\|\sum_{m=1}^M w_m \nu_m^{\bigotimes l} - \sum_{m=1}^M \hat{w}_m \hat{\nu}_m^{\bigotimes l}\|_{\infty} \le \delta$, then
    \begin{align*}
        \left \|\sum_{m=1}^M w_m \mu_m^{\bigotimes l} - \sum_{m=1}^M \hat{w}_m \hat{\mu}_m^{\bigotimes l} \right\|_{\infty} \le (2M)^{l/2} \cdot \delta.
    \end{align*}
\end{proposition}
By Proposition \ref{proposition:short_time_core_set_moments}, we can conclude that it is sufficient to estimate $T_H := \sum_{m=1}^M w_m \nu_m^{\bigotimes H}$ and find $\{(\hat{w}_m, \hat{\nu}_m)\}_{m=1}^M$ that matches $T_H$ element-wise up to accuracy $\delta := (\epsilon/H) / (Z\sqrt{2M})^{H}$. 


\subsection{Main Result} 
\label{subsec:general_lmab_final}
The sections above have outlined the key ideas we need {\em assuming knowledge of} $\mathbf{U}$. In this section, we show how we can estimate $\mathbf{U}$, and we describe the complete procedure that learns a near-optimal policy of an LMAB instance. We give the details in Algorithm~\ref{algo:learn_lmab}. Our algorithm is divided to three steps as detailed below. 

{\bf Step 1: Estimating $\mathbf{U}$}. Algorithm~\ref{algo:learn_lmab} first estimates the second-order moments $M_2 = \sum_{m=1}^M w_m \mu_m \mu_m^\top$. Let $\widehat{\mathbf{U}}$ be the top-$M$ eigenvectors of an empirical estimate of $M_2$, where $\hat{M}_2$ is constructed by collecting samples of reward correlations by taking random actions at first two time steps. Note that $\widehat{\mathbf{U}}$ may not be a good proxy for some $\mu_m$ with small mixing probability $w_m \approx 0$, and thus all elements in $\mathbf{U}$ are not necessarily close to $\widehat{\mathbf{U}}$.
We defer the details of subspace recovery procedure to the proof of the following lemma in Appendix \ref{appendix:lemma_subspace_estimation}. 
\begin{lemma}
    \label{lemma:subspace_estimation}
    Let $\widehat{\mathbf{U}}$ be a subspace spanned by top-$M$ eigenvectors of $\hat{M}_2$. After we estimate $\hat{M}_2$ using $N_0 = O(A^4 Z^2 \log(ZA/\eta) / \delta_{\mathrm{sub}}^4)$ episodes, with probability at least $1 - \eta$, for all $m \in [M]$, there exists $\Delta_m: \|\Delta_m\|_{\infty} \le \delta_{\mathrm{sub}}/ w_m^{1/2}$ such that $\mu_m + \Delta_m \in \widehat{\mathbf{U}}$. 
\end{lemma}
Our choice of $\delta_{\mathrm{sub}}$ differs in two regimes as the following:
\begin{align}
    \delta_{\mathrm{sub}} &= \frac{\epsilon}{2 Z M H^2}, & \text{if: } H \ge 2M-1, \nonumber \\
    \delta_{\mathrm{sub}} &= \frac{\min \left(\sqrt{w_{\mathrm{min}} + \epsilon / (MH^2 (Z\sqrt{2M})^{H})}, \ \epsilon/(H\sqrt{M}) \right)}{2Z \sqrt{M} H}, & \ \text{else: } H < 2M-1, \label{eq:choice_of_delta_a}
\end{align}
where $w_{\mathrm{min}} = \min_{m \in [M]} w_m$. The main difference in two regimes is that for $H\geq 2M-1$, when the parameters are identifiable, latent contexts with small mixing probabilities can be ignored once $w_m = o(\epsilon / M)$ to guarantee the closeness in distributions of observations. However, in the parameter unidentifiable regime, total variation distance is bounded only through errors in the moment space. Therefore, the estimated subspace needs to be accurate even for contexts with small mixing probabilities to keep the higher-order moments well approximated. Note that for instances with well-balanced mixing probabilities, {\it i.e.,} if $w_{\mathrm{min}} = \Omega(1/M)$, the order of $\delta_{\mathrm{sub}}$ 
remains the same as in the $H \ge 2M-1$ case. 

{\bf Step 2: Moment Matching}. Given $\widehat{\mathbf{U}}$, the subspace spanned by top-$M$ eigenvectors of $\hat{M}_2$, Algorithm~\ref{algo:learn_lmab}, follows the procedure described in Sections~\ref{subsec:exp_des}-\ref{subsec:general_short_horizon}. It constructs the set of (approximate) core action-values pairs $\{(a_j, z_j)\}_{j=1}^n$ (see Appendix \ref{appendix:theorem_lmab_procedure} for the detailed algorithm). Then, we construct higher-order moments of the core action-value pairs. For every multi-index $(i_1, i_2, ..., i_l) \in [n]^l$, using $N_1 = O(\log(l n^l/\eta) / \delta_{\mathrm{tsr}}^2)$ episodes, we execute $a_t^k = a_{i_t}$ for $t=1,...,l$ and estimate higher-order moments (as also described in equation~\eqref{eq:expectation_tensor})
\begin{align}
    \hat{T}_l(i_1, i_2, ..., i_l) = \frac{1}{N_1} \sum_{k=1}^{N_1} \Pi_{t=1}^l \indic{r_t^k = z_{i_t}}. \label{eq:tensor_estimate}
\end{align}
Using standard concentration inequalities and applying the union bounds over all elements in tensors, we get $\|\hat{T}_l - T_l\|_{\infty} < \delta_{\mathrm{tsr}}$ with probability at least $1 - \eta$. Then we find empirical parameters $\{(\hat{w}_m, \hat{\nu}_m)\}_{m=1}^M$ that satisfy 
\begin{align}
    \left\| \sum_{m=1}^M \hat{w}_m \hat{\nu}_m^{\bigotimes l} - \hat{T}_l \right\|_{\infty} < \delta_{\mathrm{tsr}}, \qquad \forall l \in [\min(H, 2M-1)]. \label{eq:tensor_error_cond}
\end{align}
We set $\delta_{\mathrm{tsr}} = O(\epsilon/(Z H^2 M^{3.5} n))^{2M-1}$ when $H \ge 2M-1$. In the parameter unidentifiable regime $H < 2M-1$, we set $\delta_{\mathrm{tsr}} = O(\epsilon/H) / (Z\sqrt{2M})^H$. 

\begin{algorithm}[t]
    \caption{} 
    \label{algo:learn_lmab}
    \begin{algorithmic}[1]
        \State{\textbf{Input:}  Accuracy level $\epsilon, \delta_{\mathrm{sub}}, \delta_{\mathrm{tsr}} >0$, model parameters $M,A,Z,H$.}
        \State \algcomment{\textbf{Step 1}: Estimate subspace $\mathbf{U}$.} 
        \State{Construct $\hat{M}_2$, an estimate of $M_2 = \sum_{m=1}^M w_m \mu_m \mu_m^\top$ using $N_0 = \tilde{O}(A^4 Z^2 / \delta_{\mathrm{sub}}^4)$ episodes.}
        \State{Calculate $\widehat{\mathbf{U}}$, the span of top-$M$ eigenvectors of $\hat{M}_2$.}
        \State \algcomment{\textbf{Step 2}: Estimate $\{(w_m,\nu_m)\}_{m=1}^M$.} 
        \State{Get core action-value pairs $\{(a_j, z_j)\}_{j=1}^n$ by calling Corollary \ref{corollary:action_event_sampling}.}
        \State{Construct $\{\hat{T}_l\}_{l=1}^{\min(H, 2M-1)}$ using $N = \tilde{O} (n^{\min(2M-1, H)} \cdot M^2 / \delta_{\mathrm{tsr}}^{2})$ episodes.}
        \State{Find valid empirical parameters $\{(\hat{w}_m, \hat{\nu}_m)\}_{m=1}^M$ satisfying \eqref{eq:tensor_error_cond}.}
        \State \algcomment{\textbf{Step 3}: Use the core-action value pairs to construct an empirical LMAB.}
        \State{Construct empirical model $\hat{\mB} = (\mA, \{\hat{w}_m\}_{m=1}^M, \{\hat{\mu}_m\}_{m=1}^M)$.}
        \State{\textbf{Output:} optimal policy of $\hat{\mB}$  by calling a planning oracle (Definition \ref{definition:planning_oracle}) for $\hat{\mB}$.}
    \end{algorithmic}
\end{algorithm}
{\bf Step 3: Constructing Empirical LMAB}. Finally,  Algorithm~\ref{algo:learn_lmab} uses the estimates $\{(\hat{w}_m, \hat{\nu}_m)\}_{m=1}^M$ to construct an empirical model $\hat{\mB} = (\mA, \{\hat{w}_m\}_{m=1}^M, \{\hat{\mu}_m)\}_{m=1}^M)$ after proper clipping and normalization. For this step to succeed, we require $\hat{w}_m$ and $\hat{\nu}_m$ to be valid parameters for the reconstruction of a valid empirical model. We state details on the recovery procedure in Appendix \ref{appendix:theorem_lmab_procedure}.

Once a valid empirical model $\hat{\mB}$ is obtained, the remaining step is to call the planning oracle that gives an $\epsilon$-approximate optimal policy for $\hat{\mB}$. We conclude this section with an end-to-end guarantee.
\begin{theorem}
    \label{theorem:learn_lmab}
    For any LMAB instance with $M$ latent contexts, with probability at least $1-\eta$, Algorithm \ref{algo:learn_lmab} returns an $\epsilon$-optimal policy given total number of episodes of at most
    \begin{align*}
        \poly(H, M, Z, A, 1/\epsilon) \cdot \log(AZ/\eta) + \poly(H,Z,M)^{2M-1} \cdot \log(M/\eta)/\epsilon^{4M-2}, &\quad \text{if } H \ge 2M-1, \\
        \poly(H, w_{\mathrm{min}}^{-1}, Z, A, 1/\epsilon) \cdot \log(AZ/\eta) + H^2 \log(M/\eta) \cdot \left(2MZ^2 \right)^{H} /\epsilon^{2}, &\quad \text{otherwise}.
    \end{align*}
\end{theorem}
Note that the dependency on $A$ is polynomial. This polynomial term is needed to control the subspace estimation error. The exponential term is derived from the closeness in moments as discussed earlier.


\begin{remark}[Continuous Rewards: Gaussian Case]
    So far we have focused on rewards with finite support $Z = O(1)$. However, some steps in the algorithm cannot be straightforwardly extended to continuous reward distributions. In Appendix \ref{appendix:Gaussian_Rewards}, we show a similar upper bound that is at most $$\poly(H, w_{\rm min}^{-1}, A, 1/\epsilon) \cdot \log(A/\eta) + \poly \left(H,M, \log(1/(\eta \epsilon)) \right)^{\min(H,M)} / \epsilon^{\min(2H+2, 4M-2)},$$ 
    assuming the rewards are Gaussian with unknown mean (see Theorem \ref{theorem:learn_lmab_gaussian} for the exact upper bound). 
\end{remark}



\section{Maximum Likelihood Implementation}
\label{section:max_likelihood}
In the previous section, we derived a procedure that learns a near-optimal policy of an LMAB instance. However, our procedure relies on matching higher-order moments, which, in general, is not suitable for practical implementation. In this section, we describe a maximum likelihood (MLE) method, motivated by our previous results. Importantly, we can use the Expectation-Maximization (EM) \cite{dempster1977maximum} heuristic to find an approximate solution to the MLE optimization problem. 

We can start from the set of core action-value pairs $\{(a_j, z_j)\}_{j=1}^n$ given by Corollary \ref{corollary:action_event_sampling}. Now for every time step $t$, we choose $i_t$ randomly from a uniform distribution over $[n]$, play action $a_{i_t}$, and observe $b_t := \indic{r_t = z_{i_t}}$. After repeating this for $N$ episodes, we formulate the log-likelihood function with parameterization $\theta = \{(w_m, \nu_m)\}_{m=1}^M$ as follows: 
\begin{align}
    l_N \left(\theta \right) := \frac{1}{N} \sum_{k=1}^N \log \left(\sum_{m=1}^M w_m \Pi_{t=1}^H (b_t \nu_m(i_t) + (1-b_t) (1 - \nu_m(i_t))) \right) \label{eq:log_likelihood}
\end{align}
To prevent the confusion with searching parameters, we use $q^*$ to denote any quantity $q$ constructed with ground truth parameters, {\it e.g.,} $\theta^* := \{(w_m^*, \nu_m^*)\}_{m=1}^M$. Let $\theta_N$ be the maximum likelihood estimator, {\it i.e.,}  $\theta_N = \text{arg} \max_{\theta \in \Theta} l_N(\theta)$ in some valid parameter set $\Theta$.

We can recover the sample-complexity guaranteed by moment-matching methods studied in the previous section. Specifically, we show that the maximum likelihood estimator with sufficiently many samples have nearly matching moments:
\begin{lemma}
    \label{lemma:likelihood_to_moments}
    Consider the maximum likelihood estimator $\theta_N = \{(\hat{w}_m, \hat{\nu}_m)\}_{m=1}^M$ with $N$ episodes for large enough $N$. If $N \ge C\cdot n^{\min(2H+1, 4M-1)} \log(N / \eta) / \delta^2$ for some sufficiently large constant $C > 0$, then with probability at least $1 - \eta$, 
    \begin{align}
        \|T_l^* - \hat{T}_l\|_{\infty} \le \delta, \qquad \forall l \in [\min(H, 2M-1)].
    \end{align} 
\end{lemma}
Therefore, by setting the same accuracy parameter $\delta$ used in Algorithm \ref{algo:learn_lmab}, we can obtain the required $\theta_N = \{(\hat{w}_m, \hat{\nu}_m)\}_{m=1}^M$ for matching moments. We only replace the moment-matching step (Step 2 in Algorithm \ref{algo:learn_lmab}) with solving the MLE optimization problem \eqref{eq:log_likelihood}. 

\paragraph{Provable Benefits of MLE} In Appendix \ref{subsec:poly_upper_bound_sep}, we also show that MLE solutions can automatically adapt to mild separation conditions. Specifically, we show that if $H = \tilde{O}(M^2 / \gamma^2)$ for some separation parameter $\gamma > 0$ (see Assumption \ref{assumption:separation}), then we can achieve the polynomial sample complexity for learning a near-optimal policy with the MLE solution $\theta_N$.


%
%

\paragraph{Implementation Details} While the log-likelihood formulation \eqref{eq:log_likelihood} is still non-convex and thus intractable, we can rely on powerful heuristics: we initialize the parameters with clustering methods \cite{arthur2007k} or spectral methods \cite{anandkumar2014tensor, azizzadenesheli2016reinforcement}, and run the EM algorithm to improve the accuracy \cite{dempster1977maximum} (regardless of conditions or guarantees). Therefore, while the sample-complexity upper bound for both methods remains the same, we benefit from the log-likelihood formulation despite the non-convexity.

Specifically, we follow the same steps in Algorithm \ref{algo:learn_lmab}. We first find the core action-event pairs from second-order moments as described (\textbf{Step 1}). Then we form the maximum log-likelihood objective \eqref{eq:log_likelihood} and find an approximate MLE solution with the EM algorithm (\textbf{Step 2}). After obtaining $\theta_N$, we recover an empirical model which we use as an input to an approximate planning oracle (\textbf{Step 3}). We compare our MLE-based implementation with experimental design (ED + MLE) to other baselines applicable to learning in LMABs ({\it e.g.,} L-UCRL + EM \cite{kwon2021rl}, tensor-decomposition methods \cite{anandkumar2014tensor, azizzadenesheli2016reinforcement, zhou2021regime}, naive-UCB \cite{azar2017minimax}). In our experimental results, we can observe that in practice, our method (ED+MLE) outperforms other baselines and the worst case guarantee given in Theorem \ref{theorem:learn_lmab}.

\subsection{Experiments}

We demonstrate our maximum likelihood estimation \eqref{eq:log_likelihood} on synthetic examples. We generate random LMAB instances with Bernoulli rewards, where the mean-reward vectors $\{\mu_m\}_{m=1}^M$ lie in a subspace of dimension roughly $4$, {\it i.e.,} $\text{rank} \left(\sum_{m=1}^M w_m \mu_m \mu_m^\top \right) \approx 4$. We compare our MLE with experimental design method~\eqref{eq:log_likelihood} (ED+MLE) to three benchmarks: 
\begin{enumerate}
    \item Naive UCB \cite{auer2002finite} without considering contexts (thus a returned policy is stationary).
    \item Tensor-decomposition methods \cite{anandkumar2014tensor, zhou2021regime}.
    \item L-UCRL with spectral initialization \cite{kwon2021rl}.
\end{enumerate}
Even when the theoretical conditions required for the success of tensor-decomposition or spectral methods do not, in practice, standard tensor-decomposition technique by \cite{anandkumar2014tensor} serves as good initialization for the EM or L-UCRL algorithms. After estimating the LMAB $\hat{\mB}$, we compute a heuristic policy using Q-MDP \cite{littman1995learning} for $\hat{\mB}$. We refer to the computed policy by Q-MDP \cite{littman1995learning} using the true LMAB model $\mB$ as the \emph{genie policy}.

\begin{figure}[t]
    \centering
    \includegraphics[width=0.8\textwidth]{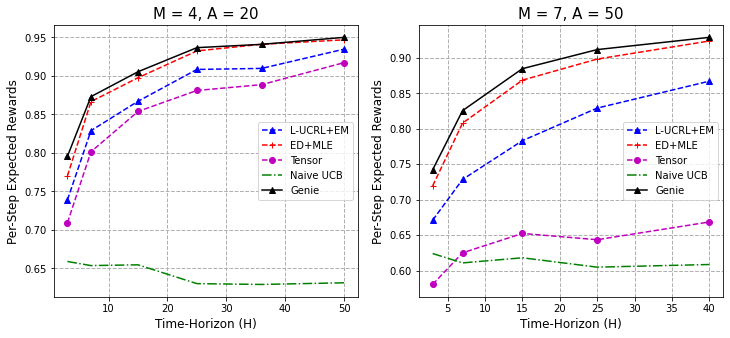}
    \caption{Per time-step rewards with history-dependent policies after model estimation}
    \label{fig:experiment_h}
\end{figure}

In the first experiment, we compare the performance of the aforementioned four alternatives. We draw a random LMAB instance with $M = 4, A = 20$ using $N = 5\cdot 10^4$ sampled episodes, and $M = 7, A = 50$ using $N = 100000$ episodes (Figure \ref{fig:experiment_h}). We compare the averaged {\it per time-step} rewards obtained with each policy with increasing length of episodes $H$. 


When $M = 4$, all methods (except naive UCB) exhibit the same pattern of improved performance as the algorithms access more data. As we generated instances to satisfy the full-rank condition, {\it i.e.,} $\text{rank} \left(\sum_{m=1}^M w_m \mu_m \mu_m^\top \right) \approx 4$, even pure tensor-decomposition method works well in practice in this setting. However, when $M = 7$, pure tensor-decomposition method significantly under-performs L-UCRL or ED+MLE. This demonstrates that for LMAB instances with rank degeneracy, additional iteration steps with EM are necessary to get a good solution. Furthermore, the performance of ED+MLE and of L-UCRL does not significantly drop in $M = 7$. This demonstrated that practically it works much better than what can be guaranteed in theory in the worst case. We conjecture that this is because the EM iteration converges to the MLE (even when converging to local optimums), and MLE solutions in general show much better performance under some mild conditions. ({\it e.g.,} with Assumption \ref{assumption:separation}, or if random perturbations are applied to the underlying LMAB model \cite{bhaskara2014smoothed}).

\begin{figure}[t]
    \centering
    \includegraphics[width=0.8\textwidth]{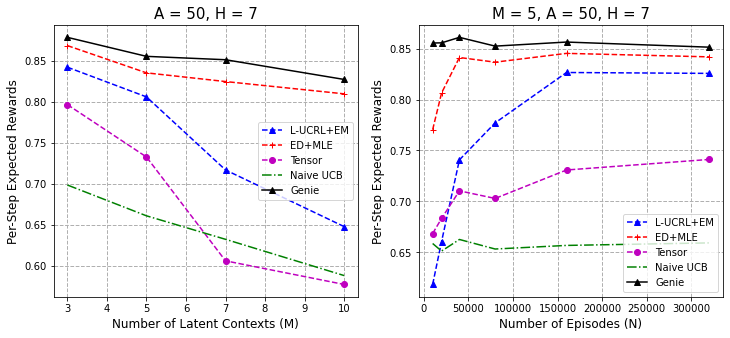}
    \caption{Left: Performance with increasing $M$, Right: Performance with increasing exploration episodes}
    \label{fig:experiment_MN}
\end{figure}

In our second experiment, we test the performance of the different methods while scanning different control parameters. We first fix $A = 50$, $H = 7$, and change the number of contexts (left of Figure \ref{fig:experiment_MN}). Since we keep generating instances satisfying $\text{rank} \left(\sum_{m=1}^M w_m \mu_m \mu_m^\top \right) \approx 4$, tensor-based method performs significantly worse as $M$ increases. We then fix $M = 5$ and observe the performance gain by increasing the sample episodes (right of Figure \ref{fig:experiment_MN}). We see that the performance of both L-UCRL+EM and our method (ED+MLE) approaches to that of genie as we increase $N$. 


From both experiments, we conclude that, in practice, MLE based methods perform much better than the worst case guarantee shown via method of moments. We believe it will be an interesting future direction whether we can derive much better guarantees for MLE based solutions.



\section{Conclusion and Future Work}
In this work, we designed an algorithm that learns a near-optimal policy for an LMAB instance that requires only $\poly(A) + \poly(M,H)^{\min(H, 2M-1)}$ number of samples. We achieved this with the experimental design and method-of-moments, and the maximum likelihood estimation which is more suitable in practice. We discuss a few limitations of this work and future directions below. 

\paragraph{Lower Bounds} The question of a lower bound on the sample complexity of the LMAB problem remains unresolved. For LMDP, an MDP extension of LMAB, a lower bound of $\Omega(A^M)$ is known due to \cite{kwon2021rl}. While we conjecture that some exponential dependence in $M$ is unavoidable, characterizing the minimax dependence of the exponent is left as an interesting future research question.

\paragraph{Latent MDPs} We believe that the moment-matching based approach we took in this work offers a promising way for designing RL algorithms in the presence of latent contexts. Specifically, we believe these techniques can be used to design algorithm that finds a near-optimal policy of LMDPs~\cite{kwon2021rl} or reward-mixing MDPs~\cite{kwon2021reinforcement} with $M = O(1)$ and without further separation assumptions.

\paragraph{Linear Bandits / Continuous Rewards} Our work has focused on the tabular setting, where all arms are independent. It will be an interesting future work to consider the same objective in linear bandit settings. Also, while we only considered Gaussian rewards, investigating a general parametric classes of rewards is an important future research direction from practical perspective. 

\bibliographystyle{abbrv}
\bibliography{main}

\begin{appendices}

\section{Proofs for Section \ref{section:main_result}}

\subsection{Proof of Lemma \ref{lemma:find_sampling_coordinates}}
\label{appendix:find_sampling_coordinates}
Let $\{\beta_j\}_{j=1}^k \subseteq \mathbb{R}^d$ be orthonormal basis of a $k$-dimensional subspace $\mathbf{U}$. Let $\Phi \in \mathbb{R}^{d\times k}$ be a matrix of form $\begin{bmatrix} \beta_1 & \beta_2 & ... & \beta_k\end{bmatrix}$, {\it i.e.,} the $j^{\mathrm{th}}$ column of $\Phi$ is $\beta_j$. We need to show the existence of a small set of core coordinates, from which we can reconstruct $\mu \in \mathbf{U}$.

Proposition \ref{proposition:core_set} implies that we can find a set of coordinates $\{i_j\}_{j=1}^n$ with cardinality at most $n = 4d \log\log d + 16$, such that if $\mu \in \mathbf{U}$ and we can access the vector  $[ \mu(i_1),\ldots,\mu(i_n) ]$ up to accuracy $\epsilon_1$, then we can reconstruct $\mu$ up to $\epsilon_1 \sqrt{2k}$ error. We can also infer that for any $\mu \in \mathbf{U}$, we have $\Phi \theta = \mu$ where $\theta$ is given by $G(\rho)^{-1} \sum_{i \in [d]} \rho(i) \mu(i) a_i$.

To conclude Lemma \ref{lemma:find_sampling_coordinates}, suppose we find $\hat{\mu} \in \mathbf{U}$ such that $|\mu(i_j) - \hat{\mu}(i_j)| \le \epsilon_0$ for all $j \in [n]$. Let $\eta$ be such that $\eta(i) = \mu(i) - \hat{\mu}(i)$ if $i \in \{i_j\}_{j=1}^n$, and $\eta(i) = 0$ otherwise. Applying Proposition \ref{proposition:core_set} with $\Delta = 0$ and $\eta$ be as defined above, with $\hat{\mu} = \Phi \hat{\theta}$, we have $\|\hat{\mu} - \mu\|_{\infty} \le \epsilon_0 \sqrt{2k}$.

\subsection{Proof of Corollary \ref{corollary:action_event_sampling}}
Let $\{\beta_m\}_{m=1}^M \subseteq \mathbb{R}^{AZ}$ be orthonormal basis of a $M$-dimensional subspace that includes $\mathbf{U} = span(\{\mu_m\}_{m=1}^M)$. 
Then by Lemma \ref{lemma:find_sampling_coordinates}, we have
\begin{align*}
    \max_{a \in \mA} \|\mu_m(a,\cdot) - \hat{\mu}_m (a,\cdot) \|_1 &= \max_{a \in \mA} \sum_{z \in \mZ} |\mu_m(a, z) - \hat{\mu}_m(a, z)| \\
    &\le 2 Z \sqrt{2M} \max_{j \in [n]} |\mu_m (a_j, z_j) - \hat{\mu}_m (a_j, z_j)|,
\end{align*}
with a set of core action-value pairs $\{(a_j, z_j)\}_{j=1}^n$. Plugging $v_m(j) = \mu_m(a_j, z_j)$ and $\hat{\nu}_m(j) = \hat{\mu}_m (a_j, z_j)$ gives Corollary \ref{corollary:action_event_sampling}.

\begin{remark}[Eliminating factor $Z$ for $H\geq M$]
    \label{remark:expensive_core_set}
    Instead of core action-value pairs, we can also define core action-event pairs to save a factor of $Z$. For instance, define a basis $\{\tilde{\beta}_m\}_{m=1}^{M}$ in a lifted space $\mathbb{R}^{A \times 2^Z}$ defined as the following:
    \begin{align*}
        \tilde{\beta}_m (a, S) := \sum_{z \in S} \beta_m(a,z), \qquad \forall a \in \mA, \ S \subseteq \mZ.
    \end{align*}
    For each $\mu_m$, define $\phi_m$:
    \begin{align}
        \label{eq:define_phi}
        \phi_m(a, S) := \sum_{z \in S} \mu_m(a,z), \qquad    \forall a \in \mA, S \subseteq \mZ.
    \end{align}
    By definition, $span(\{\phi_m\}_{m=1}^M) \subseteq span(\{\tilde{\beta}_m\}_{m=1}^M)$. Since $\frac{1}{2} \|\mu_m(a, \cdot) - \hat{\mu}_m (a, \cdot) \|_1 = \max_{S \subseteq \mZ} |\phi_m(a, S) - \hat{\phi}_m(a,S)|$, it follows from Lemma \ref{lemma:find_sampling_coordinates} that we can find a set of core action-events $\{(a_j, \mZ_j)\}_{j=1}^n$ of size at most $n = 4M \log\log M + 16$ such that for any $\phi_m, \hat{\phi}_m \in V$, we have
    \begin{align*}
        \| \phi_m - \hat{\phi}_m \|_{\infty} \le \sqrt{2M} \cdot \max_{j \in [n]} |\phi_m (a_j, \mZ_j) - \hat{\phi}_m (a_j, \mZ_j)|.
    \end{align*}
    This approach enables to approximate reward distributions in $l_1$ statistical distance, and thus we can save a factor of $Z$ in subsequent analysis. However, it comes with exponentially more expensive (in $Z$) computations for finding the core set.  
\end{remark}

\subsection{Proof of Lemma \ref{lemma:moment_closeness}}
To prove this result, we use the recent results on converting closeness in higher-order moments to closeness in Wasserstein distances for atomic distributions \cite{wu2020optimal, doss2020optimal}. A key result in \cite{wu2020optimal} states a connection between moments and Wasserstein distance in one-dimensional case:
\begin{theorem}[Proposition 1 in \cite{wu2020optimal}]\label{thm:1d_mixture}
    If $n=1$ and $\left|\sum_{m=1}^M w_m \nu_m^d - \sum_{m=1}^M \hat{w}_m \hat{\nu}_m^d \right| < \delta$ for all $d = 1, 2, ..., 2M-1$, then 
    \begin{align*}
        W(\gamma, \hat{\gamma}) \le O \left(M \delta^{1/(2M-1)} \right).
    \end{align*}
\end{theorem}
Theorem~\ref{thm:1d_mixture}, which holds for 1-dimensional mixtures--can be generalized to the high-dimensional case, as the following theorem shows.
\begin{theorem}
    For $n \ge 2$, if $\left \|\sum_{m=1}^M w_m \nu_m^d - \sum_{m=1}^M \hat{w}_m \hat{\nu}_m^d \right\|_{\infty} < \delta$ for all $d = 1, 2, ..., 2M-1$, then 
    \begin{align*}
        W(\gamma, \hat{\gamma}) \le O\left( M^3 n \cdot \delta^{1/(2M-1)} \right).
    \end{align*}
\end{theorem}
\begin{proof}
    The idea follows the proof of Lemma 3.1 in \cite{doss2020optimal}. Suppose a standard Gaussian random variable $\theta \sim \mathcal{N}(0, I)$ in $\mathbb{R}^n$. For any $x \in \mathbb{R}^n$, anti-concentration of Normal distribution says 
    \begin{align*}
        \PP(|\theta^\top x| \le \tau \|x\|_2) \le \tau,
    \end{align*}
    for any $\tau > 0$. By union bound and the fact that $\|x\|_2 \ge \|x\|_{\infty}$, if we define $\mathcal{X} = \{v - \hat{v} | v \in \{\nu_m\}_{m=1}^M, \hat{v} \in \{\hat{\nu}_m\}_{m=1}^M \}$, then we have 
    \begin{align*}
        \PP(|\theta^\top x| \le \tau \|x\|_{\infty}) \le M^2 \tau, \qquad \forall x \in \mathcal{\mathcal{X}}. 
    \end{align*}
    Also with high probability, we have $\PP(\|\theta\|_1 \ge 2n) \le \frac{n \Exs_{t \sim \mathcal{N}(0,1)} [|t|]}{2n} < 1/2$ by Markov inequality. Thus, 
    \begin{align*}
        \PP \left(\frac{\theta^\top x}{\|\theta\|_1} > \frac{\tau}{2n} \|x\|_{\infty} \right) > 1/2 - M^2\tau.
    \end{align*}
    By setting $\tau = M^2 / 2$, this probabilistic argument implies that there exists $\theta \in \mathbb{R}^n$ with a unit $L_1$ norm, $\|\theta\|_1 = 1$, such that for all $x \in \mathcal{X}$, we have $\|x\|_{\infty} \le 4M^2n |\theta^\top x|$. Using this $\theta$, we have
    \begin{align}
        \label{eq:wasserstein_multi_relation}
        W(\gamma, \hat{\gamma}) = \inf_{\Gamma} \{\Exs_{(m,m') \sim \Gamma} \left[ \|\nu_m - \hat{\nu}_{m'}\|_{\infty} \right]\} \le 4M^2 n \cdot \inf_{\Gamma} \{\Exs_{(m,m')\sim \Gamma} [ |\theta^\top (\nu_m - \hat{\nu}_{m'})| ]\}. 
    \end{align}
    
    On the other hand, consider 1-dimensional $M$-atomic distributions $\gamma_{\theta} := \sum_{m=1}^M w_m \delta_{\theta^\top \nu_m}$ and $\hat{\gamma}_{\theta} := \sum_{m=1}^M \hat{w}_m \delta_{\theta^\top \hat{\nu}_m}$. Then by Cauchy-Schwartz inequality 
    \begin{align*}
        \left|\sum_{m=1}^M w_m (\theta^\top \nu_m)^d - \sum_{m=1}^M \hat{w}_m (\theta^\top \hat{\nu}_m)^d \right| \le \|\theta\|_1^d \left\|\sum_{m=1}^M w_m \nu_m^{\bigotimes d} - \sum_{m=1}^M \hat{w}_m \hat{\nu}_m^{\bigotimes d} \right\|_{\infty} \le \delta,
    \end{align*}
    for all $d = 1, 2, ..., 2M-1$. Since $\gamma_\theta, \hat{\gamma}_\theta$ are 1-dimensional atomic distributions, by Theorem~\ref{thm:1d_mixture},
    \begin{align*}
        W(\gamma_{\theta}, \hat{\gamma}_{\theta}) := \inf_{\Gamma} \{\Exs_{(m,m') \sim \Gamma} [ |(\theta^\top \nu_m) - (\theta^\top \hat{\nu}_{m'})| ]\} \le O \left(M \delta^{1/(2M-1)} \right). 
    \end{align*}
    Note that any coupling over $\nu_{\theta}, \hat{\nu}_{\theta}$ can be converted to a coupling over $\nu, \hat{\nu}$ and vice versa. Plugging the above result into \eqref{eq:wasserstein_multi_relation}, we have the theorem. 
\end{proof}

With the above theorems, if we set $\delta < (\epsilon / (M^3n))^{2M-1}$, we have $W(\gamma, \hat{\gamma}) \le O(\epsilon)$.

\subsection{Proof of Proposition \ref{proposition:total_variation_diff}}
First note that the difference in expected values of a policy is bounded by total variation distance between trajectory distribution:
\begin{align*}
    &|V(\pi) - \hat{V}(\pi)| \\
    &\le H \cdot \sum_{\substack{r_{1:H} \in \mR^H, \\ a_{1:H} \in \mA^H}} |\PP^\pi (r_{1:H}, a_{1:H}) - \hat{\PP}^\pi (r_{1:H}, a_{1:H})| \\
    &\le H \sum_{r_{1:H}, a_{1:H}} \left|\sum_{m=1}^M w_m \Pi_{t=1}^H \mu_m(a_t, r_t) - \sum_{m=1}^M \hat{w}_m \Pi_{t=1}^H \hat{\mu}_m (a_t, r_t)\right| \cdot \Pi_{t=1}^H \pi(a_t | a_{1:t-1}, r_{1:t-1}).
\end{align*}
Now we only need to focusing on total variation distance bound. Fron this point, when we sum over sequences, we sum over all possible sequences if not specified. With a slight abuse in notation, we use a compact notation for probability of action sequences:
\begin{align}
    \pi(a_{1:t}|r_{1:t-1}) \coloneqq \Pi_{t'=1}^t \pi(a_{t'} | r_{1:t'-1}, a_{1:t'-1}).  \label{eq:simplify_action_prob}
\end{align}
First, suppose any coupling $\Gamma$ between contexts $m$ and $m'$ in two systems, and we can write 
\begin{align*}
    \sum_{r_{1:H}, a_{1:H}} & \left|\sum_{m=1}^M w_m \Pi_{t=1}^H \mu_m(a_t, r_t) - \sum_{m=1}^M \hat{w}_m \Pi_{t=1}^H \hat{\mu}_m (a_t, r_t)\right| \cdot \pi(a_{1:H}|r_{1:H-1}) \\
    &\le \sum_{r_{1:H}, a_{1:H}} \pi(a_{1:H}|r_{1:H-1}) \sum_{(m,m')} \Gamma(m,m') \left|\Pi_{t=1}^H \mu_m(a_t, r_t) - \Pi_{t=1}^H \hat{\mu}_{m'} (a_t, r_t) \right|,
\end{align*}
where the inequality holds by the triangle inequality.
Now we can proceed as
\begin{align*}
    &\sum_{(m,m')} \Gamma(m,m') \sum_{r_{1:H}, a_{1:H}} \pi(a_{1:H}|r_{1:H-1}) \left|\Pi_{t=1}^H \mu_m(a_t, r_t) - \Pi_{t=1}^H \hat{\mu}_{m'} (a_t, r_t) \right| \\
    &\le \sum_{(m,m')} \Gamma(m,m') \cdot \Bigg( \sum_{r_{1:H-1}, a_{1:H}} \pi(a_{1:H}|r_{1:H-1}) \cdot \Pi_{t=1}^{H-1} \mu_m(a_t, r_t) \sum_{r_H} |\mu_m (r_H; a_H) - \hat{\mu}_{m'} (a_H, r_H)| \\
    &\quad + \sum_{r_{1:H-1}, a_{1:H}} \pi(a_{1:H}|r_{1:H-1}) \cdot \left|\Pi_{t=1}^{H-1} \mu_m(a_t, r_t) - \Pi_{t=1}^{H-1} \hat{\mu}_{m'} (a_t, r_t)\right| \cdot \sum_{r_H} |\hat{\mu}_{m'} (a_H, r_H) | \Bigg),
\end{align*}
where we used triangle inequality at the $t = H$ time step. Note that $\sum_{r_H} |\mu_m (a_H, r_H) - \hat{\mu}_{m} (a_H, r_H)| = \|\mu( a_H, \cdot) - \hat{\mu}_m (a_H, \cdot)\|_1$. Thus, we can bound the first term by 
\begin{align*}
    \sum_{r_{1:H-1}, a_{1:H}} &\pi(a_{1:H}|r_{1:H-1}) \cdot \Pi_{t=1}^{H-1} \mu_m(a_t, r_t) \sum_{r_H} |\mu_m (a_H, r_H) - \hat{\mu}_{m'} (a_H, r_H)| \\
    &\le \max_{a \in \mA} \|\mu_m(a, \cdot) - \hat{\mu}_{m'} (a, \cdot)\|_1.
\end{align*}
where we summed over all probabilities over sequences where we used the fact that 
\begin{align*}
    \sum_{r_{1:H-1}, a_{1:H}} &\pi(a_{1:H}|r_{1:H-1}) \cdot \Pi_{t=1}^{H-1} \mu_m(a_t, r_t) =\sum_{r_{1:H-1}, a_{1:H}} \PP_m^\pi (r_{1:H-1},a_{1:H}) = 1.
\end{align*}

For the second term, we observe that
\begin{align*}
    \sum_{r_{1:H-1}, a_{1:H}} &\pi(a_{1:H}|r_{1:H-1}) \left|\Pi_{t=1}^{H-1} \mu_m(a_t, r_t) - \Pi_{t=1}^{H-1} \mu_{m'} (a_t, r_t)\right| \sum_{r_H} |\hat{\mu}_{m'} (a_H, r_H) | \\
    &= \sum_{r_{1:H-1}, a_{1:H-1}} \pi(a_{1:H-1}|r_{1:H-2}) \left|\Pi_{t=1}^{H-1} \mu_m(a_t, r_t) - \Pi_{t=1}^{H-1} \hat{\mu}_{m'} (a_t, r_t)\right|.
\end{align*}
since $\sum_{r_H} \hat{\mu}_{m'} (a, r_H) = 1$ for any $a$. From here, we can apply the same decomposition which we used at $t=H$ level, and apply the same argument recursively until $t=1$. That gives for any coupling $\Gamma$, 
\begin{align*}
    \sum_{r_{1:H}, a_{1:H}} & \left|\sum_{m=1}^M w_m \Pi_{t=1}^H \mu_m(a_t, r_t) - \sum_{m=1}^M \hat{w}_m \Pi_{t=1}^H \hat{\mu}_m (a_t, r_t)\right| \cdot \pi(a_{1:H}|r_{1:H-1}) \\
    &\le H \cdot \sum_{(m,m')} \Gamma(m, m') \max_{a \in \mA} \|\mu_m(a, \cdot) - \hat{\mu}_{m'} (a, \cdot)\|_1. 
\end{align*}
Since the inequality holds for all valid couplings $\Gamma(m, m')$ such that 
\begin{align*}
    &\sum_{m}\Gamma(m, m') = \hat{w}_{m'}, \text{and},\ \sum_{m'}\Gamma(m, m') = w_m,
\end{align*}
we can take the infimum over $\Gamma$, and conclude that
\begin{align*}
    |V(\pi) - \hat{V}(\pi)| \le H^2 \cdot \inf_{\Gamma} \sum_{(m,m')} \Gamma(m, m') \max_{a \in \mA} \|\mu_m(a, \cdot) - \hat{\mu}_{m'} (a, \cdot)\|_1.
\end{align*}

\subsection{Proof of Lemma \ref{lemma:short_time_moment_to_values}}
In order to bound difference between expected long-term rewards for a fixed history-dependent policy $\pi$, it is sufficient to bound the difference in distributions of observations following $\pi$. We first explicitly write down total variation distance between observations from $\{(w_m, \mu_m)\}_{m=1}^M$ and $\{(\hat{w}_m, \hat{\mu}_m)\}_{m=1}^M$:
\begin{align*}
    \sum_{r_{1:H}, a_{1:H}} &| \PP^\pi (r_{1:H}, a_{1:H}) - \hat{\PP}^\pi (r_{1:H}, a_{1:H}) | \\
    &= \sum_{r_{1:H}, a_{1:H}} \left| \sum_{m=1}^M w_m \Pi_{t=1}^H \mu_m(a_t, r_t) - \sum_{m=1}^M \hat{w}_m \Pi_{t=1}^H \hat{\mu}_m(a_t, r_t) \right| \pi(a_{1:H} | r_{1:H-1}) \\
    &\le \delta \sum_{r_{1:H}, a_{1:H}} \Pi_{t=1}^H \pi(a_t | a_{1:t-1}, r_{1:t-1}) = Z^H \cdot \delta,
\end{align*}
where $\pi(a_{1:H} | r_{1:H-1})$ is as defined in \eqref{eq:simplify_action_prob}, in the first inequality we used the condition $\|\sum_{m=1}^M w_m \mu_m^{\bigotimes H} - \sum_{m=1}^M \hat{w}_m \hat{\mu}_m^{\bigotimes H}\|_{\infty} \le \delta$, and the last inequality follows from that:
\begin{align*}
    \sum_{r_{1:H}, a_{1:H}} &\pi(a_{1:H} | r_{1:H-1}) = \sum_{r_{1:H}} \sum_{a_{1:H}} \Pi_{t=1}^H \pi(a_t | a_{1:t-1}, r_{1:t-1}) \\
    &= \sum_{r_{1:H} \in \mR^H} \sum_{a_{1:H-1} \in \mA^{H-1}} \Pi_{t=1}^{H-1} \pi(a_t | a_{1:t-1}, r_{1:t-1}) \sum_{a_H \in \mA} \pi(a_H | a_{1:H-1}, r_{1:H-1}) \\
    &= \sum_{r_{1:H} \in \mR^H} \sum_{a_{1:H-1} \in \mA^{H-1}} \Pi_{t=1}^{H-1} \pi(a_t | a_{1:t-1}, r_{1:t-1}) = ... = \sum_{r_{1:H} \in \mR^H} 1 = Z^H.
\end{align*}
Now the difference in expected rewards follows as
\begin{align*}
    |f(\pi) - \hat{f} (\pi)| \le H \cdot \sum_{r_{1:H}, a_{1:H}} | \PP^\pi (r_{1:H}, a_{1:H}) - \hat{\PP}^\pi (r_{1:H}, a_{1:H}) | \le H Z^H \cdot \delta.
\end{align*}

\subsection{Proof of Proposition \ref{proposition:short_time_core_set_moments}}
\label{appendix:short_horizon_core_set_moments}
To show this result, we first express each coordinate of $\mu_m$ in terms of $\nu_m$. That is, for all $(a,z) \in \mA \times \mZ$, using Proposition \ref{proposition:core_set}, by setting $\eta = 0$ we can express $\mu_m$ as 
\begin{align*}
    \mu_m = T \nu_m,
\end{align*}
for some linear mapping $T \in \mathbb{R}^{AZ \times n}$ (see Section \ref{appendix:theorem_lmab_procedure} for details on $T$). Furthermore, from the conclusion of Proposition \ref{proposition:core_set} which implies the robustness of $\mu_m$ against the perturbation of $\nu_m$ in $l_\infty$ norm, we can infer that every row of $T$ has $l_1$ norm bounded by $\sqrt{2M}$, {\it i.e.,} $\| T_{(a,z), :} \|_1 \le \sqrt{2M}$ for all $(a,z) \in \mA \times \mZ$. 

To bound $\| \sum_{m=1}^M w_m \mu_m^{\bigotimes l} - \sum_{m=1}^M \hat{w}_m \hat{\mu}_m^{\bigotimes l} \|_{\infty}$, we only need to check one entry of tensors at any position $((a_1,z_1), ..., (a_l, z_l))$ and all other entries are bounded in a similar fashion. We first check that 
\begin{align*}
    \sum_{m=1}^M &w_m \cdot \Pi_{i=1}^l \mu_m(a_i, z_i) - \sum_{m=1}^M \hat{w}_m \cdot \Pi_{i=1}^l \hat{\mu}_m(a_i, z_i) \\
    &= \sum_{m=1}^M w_m \cdot \Pi_{i=1}^l \left( \sum_{j=1}^n \nu_m(j) T_{(a_i, z_i), j} \right) - \sum_{m=1}^M \hat{w}_m \cdot \Pi_{i=1}^l \left( \sum_{j=1}^n \hat{\nu}_m(j) T_{(a_i, z_i), j} \right),
\end{align*}
Unfolding the product expression over $i$ for the original parameter part $\nu$, 
\begin{align*}
    \sum_{m=1}^M w_m \cdot \Pi_{i=1}^l \left( \sum_{j=1}^n \nu_m(j) T_{(a_i, z_i), j} \right) &= \sum_{m=1}^M w_m \sum_{(j_1, j_2, ..., j_l) \in [n]^l} \Pi_{k=1}^l \nu_m(j_k) \cdot \Pi_{k=1}^l T_{(a_k, z_k), j_k} \\
    &= \sum_{(j_1, j_2, ..., j_l) \in [n]^l}  \Pi_{k=1}^l T_{(a_k, z_k), j_k} \cdot \sum_{m=1}^M w_m \Pi_{k=1}^l \nu_m(j_k).
\end{align*}
Plugging this expression, we conclude the proof:
\begin{align*}
    \Bigg|\sum_{m=1}^M &w_m \cdot \Pi_{i=1}^l \mu_m(a_i, z_i) - \sum_{m=1}^M \hat{w}_m \cdot \Pi_{i=1}^l \hat{\mu}_m(a_i, z_i) \Bigg| \\
    &= \Bigg| \sum_{(j_1, j_2, ..., j_l) \in [n]^l} \Pi_{k=1}^l T_{(a_k, z_k), j_k} \cdot \left(\sum_{m=1}^M w_m \Pi_{k=1}^l \nu_m(j_k) - \sum_{m=1}^M \hat{w}_m \Pi_{k=1}^l \hat{\nu}_m(j_k) \right) \Bigg| \\
    &\le \delta \cdot \sum_{(j_1, j_2, ..., j_l) \in [n]^l} \Bigg|  \Pi_{k=1}^l T_{(a_k, z_k), j_k} \Bigg| = \delta \cdot \Pi_{k=1}^l \left( \sum_{j=1}^n | T_{(a_k, z_k), j} | \right) \le \delta (2M)^{l/2},
\end{align*}
which is what we needed to show.

\subsection{Proof of Lemma \ref{lemma:subspace_estimation}}
\label{appendix:lemma_subspace_estimation}
For each episode $k$ where $k \in [N_0]$, let the first and second actions be $a^k_1, a^k_2 \sim \text{Unif}(\mA)$, and let $r_1^k, r_2^k$ be the corresponding reward feedback. We construct an empirical second-order moments $\hat{M}_2$ such that $\hat{M}_2(i,j)$ is the mean of $r_1  \cdot r_2$ when $a_1^k = a_i, a_2^k = a_j$. Specifically, we construct $\hat{M}_2$ as the following:
\begin{align}
    \hat{M}_2 = \frac{1}{2 N_0} \sum_{k=1}^{N_0} \bm{e}_{(a^k_1, r^k_1)} \cdot \bm{e}_{(a^k_2, r^k_2)}^\top + \bm{e}_{(a^k_2, r^k_2)} \cdot \bm{e}_{(a^k_1, r^k_1)}^\top, \label{eq:subspace_estimation}
\end{align} 
where $\bm{e}_{(a^k_t, r^k_t)}$ is a standard basis vector in $\mathbb{R}^{AZ}$ with 1 at position $(a_t^k, r_t^k)$. The argument follows from a rather standard concentration argument for dimensionality reduction (e.g., Lemma 3.5 in \cite{doss2020optimal}). Let $\delta = \|M_2 - \hat{M}_2\|_2$. With standard measure of concentration arguments, we can show that $\|M_2 - \hat{M}_2\|_\infty < C \cdot \sqrt{A^2 \log(AZ / \eta) / N_0}$ with probability at least $1 - \eta$, which is translated to $\delta \le C \cdot \sqrt{\frac{A^4 Z^2  \log(AZ/\eta)}{N_0}}$ for some universal constant $C > 0$. 

Let $P_{\mathbf{U}}$ be the orthogonal projector onto the top-$M$ eigenspace of $M_2$ and $P_{\mathbf{U}}^{\perp} = I -  P_{\mathbf{U}}$. We can also define similar quantities from the empirical estimate $\hat{M}_2$. Let $P_{\widehat{\mathbf{U}}}$ similarly be the orthogonal projector onto the top-$M$ eigenspace of $\hat{M}_2$, and let $P_{\widehat{\mathbf{U}}^{\perp}} = I - P_{\widehat{\mathbf{U}}}$. By Weyl's theorem, we have $\lambda_{M+1}(\hat{M}_2) \le \delta$ since ${\rm rank} (M_2)$ is $M$. Our goal is to bound 
\begin{align*}
    \| \mu_m - P_{\widehat{\mathbf{U}}} \mu_m \|_2^2 &= \| P_{\widehat{\mathbf{U}}^\perp} \mu_m \|_2^2 = \max_{y \in Range(\widehat{\mathbf{U}}^\perp) \cap \mathbb{S}^{AZ-1}} (y^\top \mu_m)^2.
\end{align*}
for any $m \in [M]$. Let $y$ be the vector maximizing the above. Since $w_m \mu_m \mu_m^\top \preceq M_2$ and since $y$ belongs to the eigenspace of ranks lower than $M$, 
\begin{align*}
    w_m y^\top \mu_m \mu_m^\top y \le y^\top M_2 y = y^\top (M_2 - \hat{M}_2)y + y^\top \hat{M}_2 y \le \|M_2 - \hat{M}_2\|_2 + \delta \le 2\delta.
\end{align*}
We have shown that $\|\mu_m - \widehat{\mathbf{U}} \mu_m\|_2 \le \sqrt{2 \delta / w_m}$. Thus we need $\delta := O(\delta_{\rm sub}^2)$ to bound the $l_2$ error by $\delta_{\rm sub} / w_m^{1/2}$. Thus $N_0 = O \left(Z^2 A^4 \log(AZ/\eta) / \delta_{\rm sub}^{4} \right)$ samples are sufficient for the estimation of the subspace $\mathbf{U}$. Note that we have not optimized for the polynomial factors which can be improved by more tightly bounding $\delta = \|M_2 - \hat{M}_2\|_2$ from $N_0$ samples.

\subsection{Deferred Details of Theorem \ref{theorem:learn_lmab}}
\label{appendix:theorem_lmab}
The proof of Theorem \ref{theorem:learn_lmab} follows by collecting the results from the preceding lemmas for subspace estimation (Lemma \ref{lemma:subspace_estimation}) and results from Section \ref{subsec:general_long_horizon} and \ref{subsec:general_short_horizon}. Before we get into the proof, let us describe some details about finding a set of core action-value pairs, and how to recover the original model parameters $\{\mu_m\}_{m=1}^M$ from $\{\nu_m\}_{m=1}^M$.

\subsubsection{Detailed Procedures for Experimental Design}
\label{appendix:theorem_lmab_procedure}


\paragraph{Step 1. Subspace estimation} We first find the set of core action-value pairs $\{(a_j, z_j)\}_{j=1}^n$ following the same procedure in Corollary \ref{corollary:action_event_sampling}. By Lemma \ref{lemma:subspace_estimation}, for every $\mu_m$, we have
\begin{align*}
    \mu_m + \Delta_m \in \widehat{\mathbf{U}},
\end{align*}
where $\|\Delta_m\|_{\infty} \le \delta_{\mathrm{sub}} / w_m^{1/2}$. 

\paragraph{Step 2. Pick core action-event pairs} Let $\{\hat{\beta}_j\}_{j=1}^M$ be the orthonormal basis of $\widehat{\mathbf{U}}$. We can use this basis as input to Corollary \ref{corollary:action_event_sampling} to get a set of (approximate) core action-event pairs. Specifically, let $\hat{\Phi} \in \mathbb{R}^{AZ \times M}$ be a feature matrix where each $j^{\mathrm{th}}$ column $\hat{\Phi}_{:,j}$ is given as:
\begin{align}
    \hat{\Phi}_{:,j} (a, z) := \hat{\beta}_j(a,z), \qquad \forall a \in \mA, z \in \mZ.  \label{eq:feature_estimate}
\end{align}
After invoking Theorem \ref{theorem:small_core_set} with supplying $\hat{\Phi}$ as input, we use the support of $\rho$ as the set of core action-value pairs $\{(a_j, z_j)\}_{j=1}^n$. Let $\hat{G}(\rho)$ be defined as in equation \eqref{eq:g_optimal_design}. Note that with too small mixing weights $w_m$, we can instead use $\Delta_m = -\mu_m$ with $\|\Delta_m\|_{\infty} \le 1$.

\paragraph{Step 3. Search constraints for moment-matching} After finding a core action-event pairs $\{(a_j, z_j)\}_{j=1}^n$ as in Corollary \ref{corollary:action_event_sampling}, we estimate $\nu_m$ from higher-order tensors $\{\hat{T}_l\}_{l=1}^{\min(H, 2M-1)}$. When searching parameters for $\{(\hat{w}_m, \hat{\nu}_m)\}_{m=1}^M$, we can put constraints to ensure that $\hat{w}_m$ and $\hat{\nu}_m$ belong to a set of valid parameters for all $m \in [M]$: 
\begin{align}
    \sum_{m=1}^M \hat{w}_m = 1, \quad w_{\mathrm{min}} &\le \hat{w}_m, \quad 0 \preceq \hat{\nu}_m \preceq 1, \nonumber \\
    \left|\sum_{z \in \mZ} (\hat{T} \hat{\nu}_m)(a, z) - 1 \right| &\le -2Z \sqrt{M} \delta_{\mathrm{sub}}  / \hat{w}_m^{1/2},  \nonumber \\
    -2 \sqrt{M} \delta_{\mathrm{sub}} / \hat{w}_m^{1/2} \preceq \hat{T} \hat{\nu}_m &\preceq 1 + 2\sqrt{M} \delta_{\mathrm{sub}} / \hat{w}_m^{1/2}, \qquad \forall m \in [M], a \in \mA, \label{eq:valid_param_2}
\end{align}
where $\preceq$ is an element-wise inequality. That is, we want that $\hat{v}_m = \hat{T} \hat{\nu}_m$ is not too far from $\hat{\mu}_m$ after clipping and normalization. Without loss of generality, we assume that a rough estimate 
of $w_{\mathrm{min}}$ is known in advance (otherwise, we can repeat the same procedure with geometrically decreasing estimates of $w_{\mathrm{min}}$, {\it e.g.,} $1/M, 1/2M, 1/4M, ..., 1/M^{O(\min(M,H))}$, and pick the best returned policy). A solution satisfying all constraints is guaranteed to exist since the true model $\{(w_m, \nu_m)\}_{m=1}^M$ also satisfies constraints.

\paragraph{Step 4. Recovery of parameters} Let $\hat{\mu}_m$ be computed by \eqref{eq:coordinate_recovery} using $\hat{\nu}_m$ and $\hat{\Phi}$, {\it i.e.,} let $\hat{T} \in \mathbb{R}^{AZ \times n}$ be defined as
\begin{align}
    \hat{T}_{:, j} := \rho(a_j, z_j) \hat{\Phi} \hat{G}(\rho)^{-1} \hat{\Phi}_{(a_j, z_j), :} \qquad \forall j \in [n] \label{eq:transform_subspace}.
\end{align}
We let $v_m = \hat{T} \nu_m, \hat{v}_m = \hat{T} \hat{\nu}_m$. We recover $\hat{\mu}_m$ from $\hat{v}_m = \hat{T} \hat{\nu}_m$ as
\begin{align}
    \tilde{\mu}_m &:= \text{clip} (\hat{v}_m, 0, 1), \nonumber \\
    \hat{\mu}_m (a,z) &:= \frac{\tilde{\mu}_m (a,z)}{\sum_{z' \in \mZ} \tilde{\mu}_m(a,z')}, \qquad \forall a \in \mA, z \in \mZ. \label{eq:clipping_mu}
\end{align}
A simple algebra can show that the normalized estimates of $\hat{\mu}_m$ are close to $\mu_m$. Specifically, we show the following with Proposition \ref{proposition:core_set}:
\begin{align*}
    \| \mu_m - \hat{v}_m \|_{\infty} &\le
    \|\Delta_m\|_{\infty} + (\|\Delta_m\|_{\infty} + \max_{j \in [n]} |\mu_m (a_j, z_j) - \hat{\mu}_m (a_j, z_j)|) \cdot \sqrt{2M} \\
    &\le \|\Delta_m\|_{\infty} + (\|\Delta_m\|_{\infty} + \|\nu_m - \hat{\nu}_m\|_{\infty}) \cdot \sqrt{2M}.
\end{align*}

\subsubsection{Proof of Theorem \ref{theorem:learn_lmab}, Case I: $H \ge 2M-1$} 
We search the empirical parameters $\{(\hat{w}_m, \hat{\nu}_m)\}_{m=1}^M)$ over the set \eqref{eq:valid_param_2}. After finding $\hat{w}_m, \hat{\nu}_m$ that satisfy the moment matching condition \eqref{eq:tensor_error_cond}, we recover $\hat{\mu}_m$ from $\hat{v}_m = \hat{T} \hat{\nu}_m$ after clipping and normalization \eqref{eq:clipping_mu}. Then we first observe that for any $a \in \mA$,
\begin{align*}
    \|\mu_m(a, \cdot) - \tilde{\mu}_m(a, \cdot)\|_1 \le \sum_{z \in \mZ} |\mu_m(a,z) - \hat{v}_m(a,z)| \le Z \max_{z \in \mZ} |\mu_m(a, z) - \hat{v}_m(a, z)|,
\end{align*}
where in the first inequality, we used the fact that clipping can only improve the $l_1$ error. Then, errors from the normalization can be bounded as 
\begin{align*}
    \|\mu_m(a, \cdot) - \hat{\mu}_m(a, \cdot)\|_1 &\le \|\mu_m(a, \cdot) - \tilde{\mu}_m(a, \cdot)\|_1 + \|\tilde{\mu}_m(a, \cdot) - \hat{\mu}_m(a, \cdot)\|_1 \\
    &= \|\mu_m(a, \cdot) - \tilde{\mu}_m(a, \cdot)\|_1 + |\|\tilde{\mu}_m(a, \cdot)\|_1 - 1| \\
    &\le 2 \|\mu_m(a, \cdot) - \tilde{\mu}_m(a, \cdot)\|_1 \\
    &\le 2Z \max_{z \in \mZ} |\mu_m(a, z) - \hat{v}_m(a, z)| \\
    &\le 2Z \left(\|\Delta_m\|_{\infty} + (\|\Delta_m\|_{\infty} + \|\nu_m - \hat{\nu}_m\|_{\infty}) \sqrt{2M} \right).
\end{align*}
The second line holds since 
\begin{align*}
\|\tilde{\mu}_m(a,\cdot) - \hat{\mu}_m(a, \cdot)\|_1 &=  \sum_{z \in \mZ} \left|\tilde{\mu}_m(a,z) - \frac{\tilde{\mu}_m(a,z)}{\sum_{z \in \mZ} \tilde{\mu}_m(a,z)}\right| \\
&= \left|1 - \frac{1}{\sum_{z \in \mZ} \tilde{\mu}_m(a,z)}\right| \cdot \sum_{z \in \mZ} \left|\tilde{\mu}_m(a,z) \right|  \\
&= \left| \sum_{z \in \mZ} \tilde{\mu}_m(a,z) - \frac{\sum_{z \in \mZ} \tilde{\mu}_m(a,z)}{\sum_{z \in \mZ} \tilde{\mu}_m(a,z)} \right| = |\|\tilde{\mu}_m(a,z)\|_1 - 1|.
\end{align*}
where the third relation holds since $\sum_{z \in \mZ} \left|\tilde{\mu}_m(a,z) \right| /\left|\sum_{z \in \mZ} \tilde{\mu}_m(a,z) \right|=1$ since $\tilde{\mu}_m(a,z)\geq 0.$

By the choice of $\delta_{\mathrm{sub}}$, we have $\|\Delta_m\|_{\infty} \le \delta_{\mathrm{sub}} / w_m^{1/2} \le \epsilon / (2 Z MH^2 w_m^{1/2})$. Now we can call Proposition \ref{proposition:total_variation_diff}, and proceed as 
\begin{align*}
    |V(\pi) - \hat{V}(\pi)| &\le H^2 \cdot \inf_{\Gamma} \sum_{(m,m') \in [M]^2} \left( \Gamma(m, m') \cdot \max_{a \in \mA} \|\mu_m(a, \cdot) - \hat{\mu}_{m'} (a, \cdot)\|_{1} \right) \\
    &\le 2 Z H^2 \cdot \inf_{\Gamma} \sum_{(m,m') \in [M]^2} \Gamma(m, m') \cdot \left( \sqrt{M/w_m} \cdot \epsilon / (2 ZM H^2) + \sqrt{2M} \cdot \|\nu_m - \hat{\nu}_{m'}\|_{\infty}  \right) \\
    &\le 2 Z H^2 \cdot \left ( \sum_{m \in [M]} w_m \cdot \sqrt{M/w_m} \cdot \epsilon / (2 ZM H^2) + \sqrt{2M} \cdot W(\gamma, \hat{\gamma}) \right) \\
    &\le 2 Z H^2 \cdot \left( \epsilon / (2H^2) + 2\sqrt{2M} \cdot W(\gamma, \hat{\gamma}) \right),
\end{align*}
where in the last inequality, we used Cauchy-Schwartz inequality $\sum_{m=1}^M \sqrt{w_m} \le \sqrt{M}$. Hence if we have $2ZH^2 \sqrt{2M} W(\gamma, \hat{\gamma}) \le \epsilon$, which is given by the choice of $\delta_{\mathrm{sub}}$ and Lemma \ref{lemma:moment_closeness}, we have $|V(\pi) - \hat{V}(\pi)| \le O(\epsilon)$.

\subsubsection{Proof of Theorem \ref{theorem:learn_lmab}, Case II: $H < 2M-1$}
If $H < 2M-1$, we start by observing that for any $0 \preceq \hat{\nu}_m \preceq 1$, using Proposition \ref{proposition:core_set} and setting $\mu = 0$, $\epsilon_0 = 0$, $\eta = \nu_m$ and $\epsilon_1 = 1$, we have $-\sqrt{2M} \preceq \hat{T}\hat{\nu}_m \preceq \sqrt{2M}$. 

To exploit the moment-closeness property, we define an auxiliary model $\{(w_m, v_m)\}_{m=1}^M$ where $v_m := \hat{T} \nu_m$. Similarly to $H \ge 2M-1$ case, for any $a \in \mA$, we have that 
\begin{align}
    \|v_m(a,\cdot) - \mu_m(a, \cdot)\|_{1} \le 2 Z \|\Delta_m\|_{\infty}(1+\sqrt{2M}) \le Z\sqrt{2M} \delta_{\mathrm{sub}} /w_m^{1/2}. \label{eq:diff_vmu}
\end{align} 
We also have that 
\begin{align}
    \|v_m(a,\cdot)\|_1 \le \min\left( Z\sqrt{2M}\delta_{\mathrm{sub}} / w_m^{1/2}, Z\sqrt{2M} \right). \label{eq:l1_bound_vmu}
\end{align}
Let $\hat{v}_m = \hat{T} \hat{\nu}_m$ and $\hat{\mu}_m$ be defined as in \eqref{eq:clipping_mu}. Recall our goal to bound
\begin{align*}
    V(\pi) - \hat{V}(\pi) &= \sum_{a_{1:H}, r_{1:H}} \left( \sum_{t=1}^H r_t \right) \left(\sum_{m=1}^M w_m \Pi_{t=1}^H \mu_m(a_t, r_t) - \sum_{m=1}^M \hat{w}_m \Pi_{t=1}^H \hat{\mu}_m(a_t, r_t)\right) \pi(a_{1:H} | r_{1:H-1}).
\end{align*}
Define auxiliary value functions $V_{\mathrm{aux}}(\pi)$ and $\hat{V}_{\mathrm{aux}}(\pi)$ as
\begin{align*}
    V_{\mathrm{aux}}(\pi) &\coloneqq \sum_{a_{1:H}, r_{1:H}} \left( \sum_{t=1}^H r_t \right) \sum_{m=1}^M w_m \Pi_{t=1}^H v_m(a_t, r_t) \pi(a_{1:H} | r_{1:H-1}), \\
    \hat{V}_{\mathrm{aux}} (\pi) &\coloneqq \sum_{a_{1:H}, r_{1:H}} \left( \sum_{t=1}^H r_t \right) \sum_{m=1}^M \hat{w}_m \Pi_{t=1}^H \hat{v}_m(a_t, r_t) \pi(a_{1:H} | r_{1:H-1}),
\end{align*}
Then $|V(\pi) - \hat{V}(\pi)| \le |V(\pi) - V_{\mathrm{aux}} (\pi)| + |V_{\mathrm{aux}} (\pi) - \hat{V}_{\mathrm{aux}}(\pi)| + |\hat{V}_{\mathrm{aux}}(\pi) - \hat{V}(\pi)|$. We bound each term separately. 

\paragraph{Term I. $|V(\pi) - V_{\mathrm{aux}}(\pi)|$}: This is less than 
\begin{align*}
    |V(\pi) &- V_{\mathrm{aux}}(\pi)| = \left|\sum_{a_{1:H}, r_{1:H}} \left( \sum_{t=1}^H r_t \right) \sum_{m=1}^M w_m \left(\Pi_{t=1}^H \mu_m(a_t, r_t) - \Pi_{t=1}^H v_m(a_t, r_t)\right) \pi(a_{1:H} | r_{1:H-1}) \right| \\
    &\le H \cdot \sum_{a_{1:H}, r_{1:H}} \left| \sum_{m=1}^M w_m \left(\Pi_{t=1}^H \mu_m(a_t, r_t) - \Pi_{t=1}^H v_m(a_t, r_t)\right) \pi(a_{1:H} | r_{1:H-1}) \right| \\
    &\le H \cdot \sum_{a_{1:H-1}, r_{1:H-1}} \sum_{m=1}^M w_m \left| \Pi_{t=1}^{H-1} \mu_m(a_t, r_t) - \Pi_{t=1}^{H-1} v_m(a_t, r_t)\right| \sum_{a_H, r_H} |v_m(a_H, r_H)| \pi(a_{1:H} | r_{1:H-1})  \\
    &\quad + H\cdot \sum_{a_{1:H-1}, r_{1:H-1}} \sum_{m=1}^M w_m \Pi_{t=1}^{H-1} \mu_m(a_t, r_t) \sum_{a_H, r_H} \left| \mu_m(a_H, r_H) - v_m(a_H, r_H)\right| \pi(a_{1:H} | r_{1:H-1})  \\
    &\le H \sum_{m=1}^M w_m \max_{a \in \mA} \|v_m(a, \cdot)\|_1 \cdot \sum_{a_{1:H-1}, r_{1:H-1}} \left| \Pi_{t=1}^{H-1} \mu_m(a_t, r_t) - \Pi_{t=1}^{H-1} v_m(a_t, r_t) \right| \pi(a_{1:H-1} | r_{1:H-2}) \\
    &\quad + H \sum_{m=1}^M w_m \max_{a \in \mA} \|\mu_m(a, \cdot) - v_m(a, \cdot)\|_1 \cdot \sum_{a_{1:H-1}, r_{1:H-1}} \Pi_{t=1}^{H-1} \mu_m(a_t, r_t) \pi(a_{1:H} | r_{1:H-1}).
\end{align*}
For the second term, we have  
\begin{align*}
    &\left|\sum_{a_{1:H-1}, r_{1:H-1}} \Pi_{t=1}^{H-1} \mu_m(a_t, r_t) \pi(a_{1:H} | r_{1:H-1}) \right| = \left|\sum_{a_{1:H-1}, r_{1:H-1}} \PP_m(a_{1:H-1},r_{1:H-1}) \right|=1
\end{align*}
for all $m \in [M]$. For the first term, we can recursively apply the same inequality until the time step reaches to $t=1$. Applying this recursively,
\begin{align*}
    |V(\pi) - V_{\mathrm{aux}}(\pi)| &\le H^2 \cdot \sum_{m=1}^M w_m \left( \max_{a \in \mA} \|v_m(a,\cdot)\|_1 \right)^{H-1} \cdot \max_{a \in \mA} \|\mu_m(a, \cdot) - v_m(a, \cdot)\|_1.
\end{align*}
To bound the above, we first consider the case when $w_m \ge \frac{\epsilon}{H^2 M (Z\sqrt{2M})^H}$. Note that if $w_{\mathrm{min}} \ge \frac{\epsilon}{H^2 M (Z\sqrt{2M})^H}$, then this is always the case. In this case, for every $a \in \mA$, we have (recall \eqref{eq:diff_vmu})
\begin{align*}
    \max_{a \in \mA} \|\mu_m(a,\cdot) - v_m(a, \cdot)\|_1 &\le 2\sqrt{M} Z\delta_{\mathrm{sub}} /w_m^{1/2} \le \frac{\epsilon}{H^2 \sqrt{Mw_m}}, \\
    \max_{a \in \mA} \|v_m(a,\cdot)\|_1 &\le 1 + \max_{a \in \mA} \|v_m(a,\cdot) - \mu_m(a,\cdot)\|_1 \\
    &\le 1 + 2\sqrt{M} Z\delta_{\mathrm{sub}} /w_m^{1/2} \le 1 + 1/H,
\end{align*}
where we use our choice of $\delta_{\mathrm{sub}}$ in \eqref{eq:choice_of_delta_a}. By the second condition, $$(\max_{a \in \mA} \|v_m(a,\cdot)\|_1)^{H-1} \le (1+1/H)^{H-1} \le e.$$ 
The first condition can be combined with Proposition \ref{proposition:core_set} similarly to the $H \ge 2M-1$ case to get
\begin{align*}
    H^2 \sum_{m: w_m \ge \frac{\epsilon}{H^2 M (Z\sqrt{2M})^H}} w_m &\left( \max_{a \in \mA} \|v_m(a,\cdot)\|_1 \right)^{H-1} \max_{a\in\mA} \|\mu_m(a,\cdot) - v_m(a, \cdot)\|_1 \\
    &\le e \sum_{m=1}^M \sqrt{w_m/M} \epsilon \le O(\epsilon). 
\end{align*}
On the other hand, if $w_m < \frac{\epsilon}{H^2 M (Z\sqrt{2M})^H}$, then we can directly bound as
\begin{align*}
    H^2 \sum_{m: w_m < \frac{\epsilon}{H^2 M (Z\sqrt{2M})^H}} w_m &\left( \max_{a \in \mA} \|v_m(a,\cdot)\|_1 \right)^{H-1} \max_{a\in\mA} \|\mu_m(a,\cdot) - v_m(a, \cdot)\|_1 \\
    &\le \frac{\epsilon}{M} \sum_{m=1}^M \frac{1}{(Z\sqrt{2M})^H} (Z\sqrt{2M})^{H} \le O(\epsilon). 
\end{align*}
Thus, we have $|V(\pi) - V_{\mathrm{aux}} (\pi)| \le O(\epsilon)$.

\paragraph{Term II. $|V_{\mathrm{aux}} (\pi) - \hat{V}_{\mathrm{aux}} (\pi)|$}: We use the moment-closeness properties between $v_m$ and $\hat{v}_m$ given similarly to Lemma \ref{proposition:short_time_core_set_moments}.
\begin{lemma}
    For any given degree $l \ge 1$, if $\|\sum_{m=1}^M w_m \nu_m^{\bigotimes l} - \sum_{m=1}^M \hat{w}_m \hat{\nu}_m^{\bigotimes l}\|_{\infty} \le \delta$, then $v_m$ and $\hat{v}_m$ satisfy
    \begin{align*}
        \left \|\sum_{m=1}^M w_m v_m^{\bigotimes l} - \sum_{m=1}^M \hat{w}_m \hat{v}_m^{\bigotimes l} \right\|_{\infty} \le (2M)^{l/2} \cdot \delta.
    \end{align*}
\end{lemma}
\begin{proof}
    Note that $\| \sum_{m=1}^M w_m (\hat{T} \nu_m)^{\bigotimes l} - \sum_{m=1}^M \hat{w}_m (\hat{T} \hat{\nu}_m)^{\bigotimes l}\|_{\infty} \le (2M)^{l/2} \delta$, following the same argument in Appendix \ref{appendix:short_horizon_core_set_moments}: the conclusion of Proposition \ref{proposition:core_set} also implies that the $l_1$ norm of every row in $\hat{T}$ is less than $\sqrt{2M}$, {\it i.e.,}
    \begin{align}
        \| \hat{T}_{(a,z), :} \|_1 \le \sqrt{2M}, \qquad \forall (a,z) \in \mA \times \mZ, \label{eq:hat_T_row_norm}
    \end{align}
    Lemma follows since $v_m$ is a vector consisting of partial coordinates of $\hat{T} \nu_m$. 
\end{proof}

Now we proceed as
\begin{align*}
    |V_{\mathrm{aux}}(\pi) &- \hat{V}_{\mathrm{aux}} (\pi)| = \left|\sum_{a_{1:H}, r_{1:H}} \left( \sum_{t=1}^H r_t \right) \left( \sum_{m=1}^M w_m \Pi_{t=1}^H v_m(a_t, r_t) - \sum_{m=1}^M \hat{w}_m \Pi_{t=1}^H \hat{v}_m(a_t, r_t) \right) \pi(a_{1:H} | r_{1:H-1}) \right| \\
    &\le H \cdot \sum_{a_{1:H}, r_{1:H}} \left| \sum_{m=1}^M w_m \Pi_{t=1}^H v_m(a_t, r_t) - \sum_{m=1}^M \hat{w}_m \Pi_{t=1}^H \hat{v}_m(a_t, r_t) \right| \pi(a_{1:H} | r_{1:H-1}) \\
    &\le H Z^H (2M)^{H/2} \delta_{\mathrm{tsr}}. 
\end{align*}
Choice of $\delta_{\mathrm{tsr}} = (\epsilon/H) / (\sqrt{2M} Z)^H$ for $H < 2M-1$ gives $|V_{\mathrm{aux}} (\pi) - \hat{V}_{\mathrm{aux}} (\pi)| = O(\epsilon)$.

\paragraph{Term III. $|\hat{V}_{\mathrm{aux}} (\pi) - \hat{V} (\pi)|$}: This case is almost similar to the case $|V (\pi) - V_{\mathrm{aux}} (\pi)|$.
\begin{align*}
    |\hat{V}(\pi) &- \hat{V}_{\mathrm{aux}}(\pi)| = \left|\sum_{a_{1:H}, r_{1:H}} \left( \sum_{t=1}^H r_t \right) \sum_{m=1}^M \hat{v}_m \left(\Pi_{t=1}^H \hat{\mu}_m(a_t, r_t) - \Pi_{t=1}^H \hat{\mu}_m(a_t, r_t)\right) \pi(a_{1:H} | r_{1:H-1}) \right| \\
    &\le H \cdot \sum_{a_{1:H}, r_{1:H}} \left| \sum_{m=1}^M \hat{w}_m \left(\Pi_{t=1}^H \hat{v}_m(a_t, r_t) - \Pi_{t=1}^H \hat{\mu}_m(a_t, r_t)\right) \pi(a_{1:H} | r_{1:H-1}) \right| \\
    &\le H \cdot \sum_{a_{1:H-1}, r_{1:H-1}} \sum_{m=1}^M \hat{w}_m \left| \Pi_{t=1}^{H-1} \hat{v}_m(a_t, r_t) - \Pi_{t=1}^{H-1} \hat{\mu}_m(a_t, r_t)\right| \sum_{a_H, r_H} |\hat{v}_m(a_H, r_H)| \pi(a_{1:H} | r_{1:H-1})  \\
    &\quad + H\cdot \sum_{a_{1:H-1}, r_{1:H-1}} \sum_{m=1}^M \hat{w}_m \Pi_{t=1}^{H-1} \hat{\mu}_m(a_t, r_t) \sum_{a_H, r_H} \left| \hat{\mu}_m(a_H, r_H) - \hat{v}_m(a_H, r_H)\right| \pi(a_{1:H} | r_{1:H-1})  \\
    &\le H \cdot \sum_{a_{1:H-1}, r_{1:H-1}} \sum_{m=1}^M \hat{w}_m \max_{a \in \mA} \|\hat{v}_m(a, \cdot)\|_1 \left| \Pi_{t=1}^{H-1} \hat{\mu}_m(a_t, r_t) - \Pi_{t=1}^{H-1} \hat{v}_m(a_t, r_t) \right| \pi(a_{1:H-1} | r_{1:H-2}) \\
    &\quad + H \sum_{m=1}^M \hat{w}_m \max_{a \in \mA} \|\hat{\mu}_m(a, \cdot) - \hat{v}_m(a, \cdot)\|_1 \sum_{a_{1:H-1}, r_{1:H-1}} \Pi_{t=1}^{H-1} \hat{\mu}_m(a_t, r_t) \pi(a_{1:H} | r_{1:H-1}) \\
    &\le H^2 \cdot \left(\sum_{m=1}^M \hat{w}_m \left(\max_{a \in \mA} \|\hat{v}_m(a,\cdot)\|_1 \right)^{H-1} \cdot \max_{a \in \mA} \|\hat{\mu}_m(a, \cdot) - \hat{v}_m(a, \cdot)\|_1 \right).
\end{align*}
For each $m \in [M]$, if $\hat{w}_{m} \ge \frac{\epsilon}{H^2 M (Z\sqrt{2M})^H}$, then
\begin{align*}
    \|\hat{v}_m(a, \cdot)\|_1 &\le \left|\sum_{z: \hat{v}_m(a, z) < 0} \hat{v}_m(a,z) \right| + \left|\sum_{z: \hat{v}_m(a, z) > 0} \hat{v}_m(a,z) \right| \\
    &\le 1 + 2Z\sqrt{M}\delta_{\mathrm{sub}} / \hat{w}_m^{1/2} \le 1 + 1/H,
\end{align*}
where we used the choice of $\delta_{\mathrm{sub}}$ in \eqref{eq:choice_of_delta_a}. We also need to show that $\|\hat{v}_m(a, \cdot) - \hat{\mu}_m(a,\cdot)\|_1$ is bounded. Let $\tilde{\mu}_m$ be the intermediate step after clipping $\hat{v}_m$ before normalization as in \ref{eq:clipping_mu}. Due to the third condition of \eqref{eq:valid_param_2}, clipped amount can be at most
\begin{align*}
    \|\hat{v}_m (a,\cdot) - \tilde{\mu}_m(a, \cdot)\|_1 \le 2Z\sqrt{M}\delta_{\mathrm{sub}} / \hat{w}_m^{1/2}.
\end{align*}
With the second condition of \eqref{eq:valid_param_2}, we have
\begin{align*}
    \|\tilde{\mu}_m(a, \cdot)\|_1 &= \sum_{z \in \mZ} \tilde{\mu}_m(a, z) \le \left|\sum_{z \in \mZ} \tilde{\mu}_m(a, z) - \hat{\mu}_m(a,z)\right| + \left| \sum_{z \in \mZ} \hat{\mu}_m(a,z)\right| \\
    &\le \|\hat{v}_m (a,\cdot) - \tilde{\mu}_m(a, \cdot)\|_1 + \| \hat{\mu}_m(a,\cdot) \|_1 \le 1 + 3Z\sqrt{M}\delta_{\mathrm{sub}}/\hat{w}_m^{1/2}. 
\end{align*}
Similarly, 
\begin{align*}
    \|\tilde{v}_m(a, \cdot)\|_1 &= \sum_{z \in \mR} \tilde{v}_m(a, z) \ge \left| \sum_{z \in \mR} v_m(a,z)\right| - \left|\sum_{z \in \mR} \tilde{v}_m(a, z) - v_m(a,z)\right| \\
    &\ge |\hat{\phi}_m(a,\mR)| - \|\hat{v}_m (a,\cdot) - \tilde{v}_m(a, \cdot)\|_1 \ge 1 - 3Z\sqrt{M}\delta_{\mathrm{sub}}/\hat{w}_m^{1/2}. 
\end{align*}
Therefore, we can show that 
\begin{align*}
    \|\hat{\mu}_m(a, \cdot) - \hat{v}_m(a, \cdot)\|_1 &\le \|\hat{\mu}_m(a, \cdot) - \tilde{v}_m(a, \cdot)\|_1 + \|\tilde{v}_m(a, \cdot) - \hat{v}_m(a, \cdot)\|_1 \\
    &\le \frac{|\|\tilde{v}_m(a, \cdot)\|_1 - 1|}{\|\tilde{v}_m(a, \cdot)\|_1} + 2Z\sqrt{M}\delta_a / \hat{w}_m^{1/2} \\
    &\le 8Z\sqrt{M}\delta_{\mathrm{sub}} / \hat{w}_m^{1/2},
\end{align*}
where we used $1 - 3Z\delta_{\mathrm{sub}}/\hat{w}_m^{1/2} \ge 1/2$ due to the first constraint and the choice of $\delta_{\mathrm{sub}}$. 

If $\hat{w}_m < \frac{\epsilon}{H^2 M (Z\sqrt{2M})^H}$, then we can use the fact that all $l_1$-norm of rows of $\hat{T}$ are less than $\sqrt{2M}$ (equation \eqref{eq:hat_T_row_norm}), and thus $\|\hat{v}_m(a,\cdot)\|_1 \le \sqrt{2M}$ for all $a \in \mA$. We also have that $\|\hat{\mu}_m(a,\cdot) - \hat{v}_m(a,\cdot)\|_1 \le 1 + \sqrt{2M}$. Now we plug all things together, and proceed as
\begin{align*}
    |\hat{V}(\pi) - \hat{V}_{\mathrm{aux}}(\pi)| &\le H^2 \cdot \left(\sum_{m=1}^M \hat{w}_m \left(\max_{a \in \mA} \|\hat{v}_m(a,\cdot)\|_1 \right)^{H-1} \cdot \max_{a \in \mA} \|\hat{\mu}_m(a, \cdot) - \hat{v}_m(a, \cdot)\|_1 \right) \\
    &\le e H^2 \cdot \left(\sum_{m:\hat{w}_m \ge \frac{\epsilon}{H^2 M (Z\sqrt{2M})^H}} \hat{w}_m \cdot \max_{a \in \mA} \|\hat{\mu}_m(a, \cdot) - \hat{v}_m(a, \cdot)\|_1 \right) \\
    & \quad + H^2 \cdot \left(\sum_{m:\hat{w}_m < \frac{\epsilon}{H^2 M (Z\sqrt{2M})^H}} \hat{w}_m \cdot (Z\sqrt{2M})^H \right) \\
    &\le 4 e H^2 \sum_{m=1}^M \epsilon \sqrt{\hat{w}_m/M} / H^2 + \epsilon \le O(\epsilon). 
\end{align*}
Collecting all three terms, we have $|V(\pi) - \hat{V}(\pi)| \le O(\epsilon)$. This concludes the proof of Theorem \ref{theorem:learn_lmab}.

\section{LMAB with Gaussian Rewards}
\label{appendix:Gaussian_Rewards}
So far we have focused on rewards with finite support $Z = O(1)$. In this section we consider and LMAB setting with Gaussian rewards and generalize Algorithm~\ref{algo:learn_lmab}. Indeed, some steps in the algorithm cannot be straightforwardly extended if $Z = \infty$. In this subsection, we consider a standard Gaussian reward distribution--a special case of continuous rewards--and generalize Algorithm~\ref{algo:learn_lmab} to this setting. We make the following assumption.
\begin{assumption}[Gaussian Rewards]
     \label{assumption:gaussian_rewards}
     The reward distribution conditioning on an action $a \in \mA$ in a context $m \in [M]$ is $\mathcal{N}(\mu_m(a), 1)$ for some $|\mu_m(a)| \le 1$.
\end{assumption}
Even though the rewards have infinite support, we show that the same conclusion holds as in finite-support case, {\it i.e.,} the sample-complexity is upper bounded by $O((MH/\epsilon)^{O(\min(H, M))} + \poly(A,H,M))$. Algorithms for the Gaussian case differ significantly in identifiable $H \ge 2M-1$ and unidentifiable regimes $H < 2M-1$. When $H \ge 2M-1$, there are only minor changes in the algorithm design for defining core actions and how tensors are constructed. We handle this case in Appendix \ref{appendix:algorithm_gaussian_long}. A more interesting case is the parameter unidentifiable regime, where we follow an alternative approach and discretize the support of rewards by $O(\epsilon/H^2)$-level. This approach is described in Appendix \ref{appendix:algorithm_gaussian_short}.

We can reach similar conclusions for the Gaussian rewards:
\begin{theorem}
    \label{theorem:learn_lmab_gaussian}
     Consider any LMAB with $M$ contexts under Gaussian reward Assumption \ref{assumption:gaussian_rewards}. There exists an algorithm such that with probability at least $1 - \eta$, it returns an $\epsilon$-optimal policy using a number of episodes at most 
     \begin{align*}
         \poly(H, M, A, 1/\epsilon, \log(A/\eta)) + \poly(\log(M/\eta), H, M)^{2M-1} \cdot \epsilon^{-(4M-2)}, &\quad \text{if } H \ge 2M-1, \\
         \poly(H, w_{\mathrm{min}}^{-1}, A, 1/\epsilon, \log(A/\eta)) + \poly(\log(MH/(\eta\epsilon)), H, M)^{H} \cdot \epsilon^{-(2H+2)}, &\quad \text{otherwise}.
    \end{align*}
\end{theorem}
Note that the dependency on $\epsilon$ is at most $\epsilon^{-(4M-2)}$ and smaller when $H < 2M-2$. In the parameter identifiable regime, similarly to discrete reward cases, near-optimality of returned policy comes from the closeness of latent model parameters in the Wasserstein metric. In the parameter unidentifiable regime, we first discretize the support of rewards in $O(\epsilon)$ level. Then we can apply the same procedures for handling discrete rewards as in Section \ref{subsec:general_short_horizon}. We provide the full details in Appendix \ref{appendix:Gaussian_Rewards}.

\subsection{Algorithm for Identifiable Regime $H \ge 2M-1$}
\label{appendix:algorithm_gaussian_long}
Let $\hat{M}_2 \in \mathbb{R}^{A \times A}$ be the empirical second-order moments as
\begin{align*}
    \hat{M}_2 = \frac{1}{2N_0} \sum_{k=1}^{N_0} r_1^k r_2^k \cdot \bm{e}_{a_1^k} \bm{e}_{a_2^k}^\top.
\end{align*}
Then let $\widehat{\mathbf{U}}$ be the subspace spanned by top-$M$ eigenvectors of $\hat{M}_2$. A similar conclusion to Lemma \ref{lemma:subspace_estimation} holds:
\begin{lemma}
    \label{lemma:subspace_estimation2}
    Let $\widehat{\mathbf{U}}$ be a subspace spanned by top-$M$ eigenvectors of $\hat{M}_2$. After we estimate $\hat{M}_2$ using $N_0 = O(A^4 \log(A/\eta) / \delta_{\mathrm{sub}}^4)$ episodes, with probability at least $1 - \eta$, for all $m \in [M]$, there exists $\Delta_m: \|\Delta_m\|_{\infty} \le \delta_{\mathrm{sub}}/ w_m^{1/2}$ such that $\mu_m + \Delta_m \in \widehat{\mathbf{U}}$. 
\end{lemma}

Proof of Lemma \ref{lemma:subspace_estimation2} is identical to the proof of Lemma \ref{lemma:subspace_estimation}. Let $\delta_{\mathrm{sub}} = \epsilon / (2MH^2)$. 

Similarly to finite-support reward distributions, let $\{\hat{\beta}_j\}_{j=1}^M$ be the orthonormal basis of $\widehat{\mathbf{U}}$ and construct $\hat{\Phi} \in \mathbb{R}^{A \times M}$ such that the $j^{\mathrm{th}}$ column of $\hat{\Phi}$ is $\hat{\beta_j}$, {\it i.e.,} $\hat{\Phi}_{:, j} = \hat{\beta}_j$. We invoke Theorem \ref{theorem:small_core_set} to get a set of core actions $\{a_j\}_{j=1}^n$. The main difference to the finite-support case is that we do not need to specify a corresponding event of rewards. Instead, we measure the correlation in terms of actual reward values. Specifically, for every multi-index $(i_1, i_2, ..., i_l) \in [n]^l$, using $N_1$ episodes where 
$$N_1 = O\left( \log (l n^l / \eta) \right)^l / \delta_{\mathrm{tsr}}^2.$$ 
We play $a_t^k = a_{i_t}$ for $t = 1, ..., l$ and $k \in [N_b]$, and estimate higher-order moments:
\begin{align*}
    \hat{T}_l (i_1, ..., i_l) = \frac{1}{N_1} \sum_{k=1}^{N_1} \Pi_{t=1}^d r_t^k. 
\end{align*}
We can easily verify that 
\begin{align*}
    T_l = \Exs[\hat{T}_l] = \sum_{m=1}^M w_m \nu_m^{\bigotimes l}. 
\end{align*}
We can apply the concentration of higher-order polynomials of sub-Gaussian random variables element-wise, given from the following lemma on hypercontractivity inequality:
\begin{lemma}[Hypercontractivity Inequality (Theorem 1.9 in \cite{schudy2012concentration})]
    \label{lemma:hypercontrat_ineq}
    Consider a degree-$l$ polynomial $f$ defined over a set of $N$ independent samples of zero-mean unit-variance Gaussians, such that $f(X_{1:N}) := f(X_1, ..., X_N)$. Then,
    \begin{align*}
        \PP\left(|f(X_{1:N}) - \Exs[f(X_{1:N})]| \ge \lambda \right) \le e^2 \exp\left( -\left( \frac{\lambda^2}{C \cdot var(f(X_{1:N}))} \right)^{1/l} \right),
    \end{align*}
    for some absolute constant $C > 0$.
\end{lemma}
To show the concentration of $\hat{T}_l (i_1,...,i_l)$ around $T_l(i_1, ..., i_l)$, we can apply Lemma \ref{lemma:hypercontrat_ineq} with plugging $\lambda = O(\delta_{\mathrm{tsr}})$ and $var(f) \le 2^l \cdot var(X^l) / N_1$, where $X \sim \mathcal{N}(0,1)$. Here, $f$ can be viewed as a degree-$l$ polynomial of $X_t^k := r_t^k - \mu_{m^k} (a_t^k)$ for $t = 1, \ldots, l$ and $k = 1, ..., N_1$, where $m^k$ is a latent context for the episode $k$.

Since $var(X^l) \le O(l^l)$, we need $N_1 = O(l^l \log^l \left(l n^l /\eta \right) / \delta_{\mathrm{sub}}^2)$ to make the exponent less than $\eta / (ln^l)$. Take union bound over all elements in $\hat{T}_l$ ensures that $\| T_l - \hat{T}_l \|_{\infty} \le \delta_{\mathrm{tsr}}$ with probability at least $1 - \eta$.

Now we find a set of parameters $\{(\hat{w}_m, \hat{\nu}_m)\}_{m=1}^M$ with the only constraint:
\begin{align*}
    \hat{w}_m \in \mathbb{R}_+, \ \sum_{m=1}^M \hat{w}_m = 1.
\end{align*}

We aim to find
\begin{align*}
    \left\|\sum_{m=1}^M \hat{w}_m \hat{\nu}_m^{\bigotimes l} - \hat{T}_l \right\|_{\infty} \le \delta_{\mathrm{tsr}}, \qquad \forall l \in [2M-1].
\end{align*}
Now we can construct an empirical model by recovering $\hat{\mu}_m = \hat{T}\hat{\nu}_m$ where $\hat{T} \in \mathbb{R}^{A \times M}$ is defined similarly to \eqref{eq:transform_subspace} as
\begin{align*}
    \hat{T}_{:, j} := \rho(a_j) \hat{\Phi} \hat{G}(\rho)^{-1} \hat{\Phi}_{a_j, :}, \qquad \forall j \in [n],
\end{align*}
where $\rho$ is a distribution over rows of $\hat{\Phi}$ found by Theorem \ref{theorem:small_core_set} and $\hat{G}(\rho)$ is defined as in \eqref{eq:g_optimal_design}. We do not need extra clipping and normalization steps here since any $\hat{\mu}_m$ is a valid model parameter. Now with $\{(\hat{w}_m, \hat{\mu}_m)\}$, we call the planning oracle \ref{definition:planning_oracle} and obtain an $\epsilon$-optimal policy.

\subsection{Algorithm with Short Time-Horizon $H < 2M-1$}
\label{appendix:algorithm_gaussian_short}
We first discretize possible reward values: let $\mZ = \{z_1, z_2, ..., z_Z\}$ where $z_i = -4\sqrt{ \log(H/\epsilon)} + (i-1) \cdot \epsilon/H^2$ and $Z = \lfloor 8H^2 \sqrt{\log (H/\epsilon)} / \epsilon \rfloor$. We define an auxiliary reward (pseudo) p.d.f $p_m(a, \cdot)$ for each $a$ and $m$ as the following: for all $s \in [Z-1]$,
\begin{align}
    p_m(a, r) &= \frac{H^2}{\epsilon} \int_{z_s}^{z_{s+1}} \frac{1}{\sqrt{2\pi}} \exp \left( -\frac{(x - \mu_m(a))^2}{2} \right) dx, & \forall r \in [z_s, z_{s+1}),
\end{align}
and $p_m(a,r) = 0$ for all $r \in (-\infty, z_1) \cup [z_Z, \infty)$. Define $\tilde{V}(\cdot)$ be the policy evaluation function in a (pseudo) LMAB model $\tilde{\mB} = (\mA, \{w_m\}_{m=1}^M, \{p_m\}_{m=1}^M)$:
\begin{align*}
    \tilde{V}(\pi) := \sum_{m=1}^M w_m \cdot \sum_{a_{1:H}} \int_{r_{1:H}} \left(\sum_{t=1}^H r_t \right) \Pi_{t=1}^H p_m(a_t, r_t) \pi(a_{1:H} | r_{1:H-1}) d(r_{1:H}).
\end{align*}
We first show that $p_m(a,\cdot)$ is good approximation of true reward distributions:
\begin{lemma}
    \label{lemma:auxiliary_system}
    Let $\mB$ and $\tilde{\mB}$ defined as above. Then for any history-dependent policy $\pi$, $|V(\pi) - \tilde{V}(\pi)| \le 10 \epsilon$. 
\end{lemma}
Given Lemma \ref{lemma:auxiliary_system}, we will discretize reward values are so that we can leverage the result of Section \ref{subsec:general_short_horizon} for short time-horizon. Specifically, let $\overline{\mB}$ be a model $\{(w_m, q_m)\}_{m=1}^M$ with discrete reward distributions taking values in $\mZ \cup \{0\}$, and let $q_m(a,z) = \frac{\epsilon}{H^2} \cdot p_m(a,z)$ for $z \in \mZ$ and $q_m(a,0) = \PP_m(r \notin [z_1, z_Z) | a)$. As if the underlying model is $\overline{\mB}$, we run Algorithm \ref{algo:learn_lmab} with manually modifying the observed rewards $r_t \rightarrow \overline{r}_t$:
\begin{align*}
    \overline{r}_t &= 0, &\quad \text{if } r_t < z_1 \text{ or } r_t \ge z_Z, \\
    \overline{r}_t &= z_s, &\quad \text{for some } s \in [Z-1], \text{ s.t. } r_t \in [z_s, z_{s+1}). 
\end{align*}
From actions $(a_1, ..., a_H)$ and reward observations $(\overline{r}_1, ..., \overline{r}_H)$, we can now apply Algorithm \ref{algo:learn_lmab} for the parameter unidentifiable case $H < 2M-1$.

\subsection{Proof of Theorem \ref{theorem:learn_lmab_gaussian}}
Now we are ready to prove the Theorem \ref{theorem:learn_lmab_gaussian} for Gaussian rewards. 
\subsubsection{Identifiable Regime $H \ge 2M-1$:}
By Lemma \ref{lemma:moment_closeness}, we know that $W(\gamma, \hat{\gamma}) \le O \left(M^3n \delta_{\mathrm{tsr}}^{-1/(2M-1)}\right)$ where $\nu = \sum_{m=1}^M w_m \delta_{\nu_m}$ and $\hat{\nu} = \sum_{m=1}^M \hat{w}_m \delta_{\hat{\nu}_m}$. With experimental design, by Corollary \ref{corollary:action_event_sampling}, we can ensure that
\begin{align*}
    \max_{a \in \mA} |\mu_m(a) - \hat{\mu}_{m'}(a)| \le \sqrt{2M} \|\nu_m - \hat{\nu}_{m'}\|_{\infty}. 
\end{align*}
Observe that for any $a \in \mA$, total variation distance between standard Gaussians is bounded by the distance between centers of Gaussians, {\it i.e.,} $d_{TV}(\mathcal{N} (\mu_m(a),1), \mathcal{N} (\hat{\mu}_{m'}(a),1)) \le |\mu_m(a) - \hat{\mu}_{m'}(a)|$. Using Proposition \ref{proposition:total_variation_diff}, we can show that
\begin{align*}
    |V(\pi) - \hat{V}(\pi)| &\le 2 H^2 \cdot \inf_{\Gamma} \sum_{(m,m')} \Gamma(m, m') \max_{a \in \mA} d_{TV}(\mathcal{N} (\mu_m(a),1), \mathcal{N} (\hat{\mu}_{m'}(a),1)) \\
    &\le 2 H^2 \cdot \inf_{\Gamma} \sum_{(m,m')} \Gamma(m, m') \max_{a \in \mA} |\mu_m(a) - \hat{\mu}_{m'}(a)| \\
    &\le 2 \sqrt{2M} H^2 \cdot \inf_{\Gamma} \sum_{(m,m')} \Gamma(m, m') \| \nu_m - \hat{\nu}_{m'} \|_{\infty} \\
    &\le 2 \sqrt{2M} H^2 W(\gamma, \hat{\gamma}).
\end{align*}
Plugging the choice of $\delta_{\mathrm{tsr}} = O(\epsilon/(H^2 M^{3.5} n))^{2M-1}$, this is less than $\epsilon$.

\subsubsection{Unidentifiable Regime $H < 2M-1$:}
Let us first compare the expected rewards from $\tilde{\mB}$ and $\overline{\mB}$ with any fixed policy $\pi$.
\begin{align*}
    \tilde{V}(\pi) &= \sum_{m=1}^M w_m \cdot \sum_{a_{1:H}} \int_{r_{1:H}} \left(\sum_{t=1}^H r_t \right) \Pi_{t=1}^H p_m(a_t, r_t) \pi(a_{1:H} | r_{1:H-1}) d(r_{1:H}), \\
    \overline{V}(\pi) &= \sum_{m=1}^M w_m \cdot \sum_{a_{1:H}} \sum_{\overline{r}_{1:H}} \left(\sum_{t=1}^H \overline{r}_t \right) \Pi_{t=1}^H q_m(a_t, \overline{r}_t) \pi(a_{1:H} | \overline{r}_{1:H-1}).
\end{align*}
With slight abuse in notation, let $\overline{r}_t$ be a quantized value of $r_t$. Then we can show that
\begin{align*}
    \left| \tilde{V}(\pi) - \overline{V}(\pi) \right| &\le \sum_{t=1}^H \int_{r_{1:H}: r_t \in [z_1, z_Z), \forall t \in [H]} \left| \sum_{m=1}^M w_m \sum_{a_{1:H}} (r_t - \overline{r}_t) \Pi_{t=1}^H p_m(a_t, r_t) \pi(a_{1:H} | r_{1:H-1}) d(r_{1:H}) \right| \\
    &\ + \sum_{t=1}^H \int_{r_{1:H}: r_t \notin [z_1, z_Z), \exists t \in [H]} \left| \sum_{m=1}^M w_m \sum_{a_{1:H}} \overline{r}_t \Pi_{t=1}^H q_m(a_t, r_t) \pi(a_{1:H} | r_{1:H-1}) d(r_{1:H}) \right| \\
    &\le \sum_{t=1}^H \int_{r_{1:H}: r_t \in [z_1, z_Z), \forall t \in [H]} \sum_{m=1}^M w_m \sum_{a_{1:H}} (\epsilon/H^2) \Pi_{t=1}^H p_m(a_t, r_t) \pi(a_{1:H} | r_{1:H-1}) d(r_{1:H}) \\
    &\ + |H z_Z| \cdot \PP(\exists t \in [H], \ s.t. \ r_t \notin [z_1, z_Z)) \\
    &\le \epsilon/H + |H^2 z_Z| \cdot \PP_{X \sim \mathcal{N}(0,1)}(X \ge z_Z - 1),
\end{align*}
where in the first inequality, we used $p_m(a,r) = 0$ for $r \notin [z_1, z_Z)$. Note that $\PP_{X \sim \mathcal{N}(0,1)} (|X| \ge z_Z - 1) \le (H/\epsilon)^4$ with $z_Z = 4\sqrt{\log (H/\epsilon)}$. 

Note that a system with manually discretized rewards can be described by the model $\overline{\mB}$. Assuming the returned policy $\hat{\pi}$ from Algorithm \ref{algo:learn_lmab} is $O(\epsilon)$-optimal for $\overline{\mB}$, by triangle inequality for the policy evaluation for any policy $\pi$,
\begin{align*}
    |\overline{V}(\pi) - V(\pi)| \le |\overline{V}(\pi) - \tilde{V}(\pi)| + |\tilde{V}(\pi) - V(\pi)|,
\end{align*}
we conclude that $\hat{\pi}$ is $O(\epsilon)$-optimal for $\mB$ with Gaussian rewards.

\subsection{Proof of Lemma \ref{lemma:auxiliary_system}}
We start by unfolding the expression for policy value differences. 
\begin{align*}
    |f(\pi) - \tilde{f}(\pi)| \le \sum_{m=1}^M w_m \cdot \sum_{a_{1:H}} \int_{r_{1:H}} \left(\sum_{t=1}^H |r_t| \right) \left|\Pi_{t=1}^H p_m(a_t, r_t) - \Pi_{t=1}^H g_m(a_t, r_t)\right| \pi(a_{1:H} | r_{1:H-1}) d(r_{1:H}),
\end{align*}
where $g_m(a_t, r_t) := \frac{1}{\sqrt{2\pi}} \exp\left(-(r_t - \mu_m(a_t))^2 / 2\right)$. We first rule out reward values greater than $z_Z = 4\sqrt{\log (H/\epsilon)}$. Define a set of bad reward sequences $\Eps_b = \{ r_{1:H} | \exists t \in [H],$ {\it s.t.,} $|r_t| > z_Z \}$. Then for any $t_0 \in [H]$,
\begin{align*}
    \sum_{a_{1:H}} &\int_{r_{1:H} \in \Eps_b} |r_{t_0}| \left|\Pi_{t=1}^H p_m(a_t, r_t) - \Pi_{t=1}^H g_m(a_t, r_t)\right| \pi(a_{1:H} | r_{1:H-1}) d(r_{1:H}) \\ 
    &= \sum_{a_{1:H}} \int_{r_{1:H} \in \Eps_b} |r_{t_0}| \cdot \Pi_{t=1}^H g_m(a_t, r_t)\pi(a_{1:H} | r_{1:H-1}) d(r_{1:H}) \\
    &\le \sum_{a_{1:H}} \int_{r_{1:H} \in \Eps_b \cap \{|r_{t_0}| \le z_Z\}} |r_{t_0}| \cdot \Pi_{t=1}^H g_m(a_t, r_t)\pi(a_{1:H} | r_{1:H-1}) d(r_{1:H}) \\
    &+ \sum_{a_{1:H}} \int_{r_{1:H} \in \Eps_b \cap \{|r_{t_0}| > z_Z\}} |r_{t_0}| \cdot \Pi_{t=1}^H g_m(a_t, r_t)\pi(a_{1:H} | r_{1:H-1}) d(r_{1:H}) \\
    &\le z_Z \cdot \PP_m(\Eps_b) + \sum_{a_{1:t_0}} \int_{r_{1:t_0} \in \{|r_{t_0}| > z_Z\}} |r_{t_0}| \cdot \Pi_{t=1}^{t_0} g_m(a_{t}, r_{t})\pi(a_{1:t_0} | r_{1:t_0-1}) d(r_{1:t_0}),
\end{align*}
where the last inequality results from integrating over probabilities for time steps $t_0+1, ..., H$. The last summation term can be further bounded by integrating out $t_0^{\mathrm{th}}$ time step since
\begin{align*}
    \int_{\{|r_{t_0}| > z_Z\}} |r_{t_0}| g_m(a_{t_0}, r_{t_0}) d(r_{t_0}) &= \int_{\{|x| > 4\sqrt{\log(H/\epsilon)}\}} \frac{|x|}{\sqrt{2\pi}} \exp \left(-\frac{(x - \mu_m(a_t))^2}{2} \right) dx\\
    &\le 2 \Exs_{X \sim \mathcal{N}(0,1)} \left[|X| \cdot \mathds{1} \{|X| > 3\sqrt{\log (H/\epsilon)}\} \right] \le 2\epsilon^2 / H^2,
\end{align*}
where in the first inequality we used $|\mu_m(a_t)| \le 1 \le \sqrt{\log(H/\epsilon)}$. In last inequality we used $\Exs[|X| \cdot \mathds{1} \{|X| \ge t\}] \le \sqrt{\PP(|X| \ge t)}$ by Cauchy-Schwartz inequality, and then used the Gaussian tail bound $\PP(|X| \ge t) \le \exp(-t^2/2)$ for $t = 3\sqrt{\log (H/\epsilon)}$. Therefore, we now get
\begin{align*}
    \sum_{a_{1:H}} &\int_{r_{1:H} \in \Eps_b} |r_{t_0}| \left|\Pi_{t=1}^H p_m(a_t, r_t) - \Pi_{t=1}^H g_m(a_t, r_t)\right| \pi(a_{1:H} | r_{1:H-1}) d(r_{1:H}) \\
    &\le 4\sqrt{\log(H/\epsilon)} \cdot \epsilon^4 / H^3 + 2\epsilon^2 / H \le 4\epsilon^2/H,
\end{align*}
where we used $\PP_m(\Eps_b) \le H \cdot \PP(|X| \ge z_Z - 1) \le \epsilon^4 / H^3$ with sufficiently small $\epsilon > 0$. This can be similarly done for all $t_0$, and thus
\begin{align*}
    \sum_{a_{1:H}} \int_{r_{1:H} \in \Eps_b} \left(\sum_{t=1}^H |r_t| \right) \left|\Pi_{t=1}^H p_m(a_t, r_t) - \Pi_{t=1}^H g_m(a_t, r_t)\right| \pi(a_{1:H} | r_{1:H-1}) d(r_{1:H}) \le 4\epsilon^2.
\end{align*}

We remain to bound
\begin{align*}
    \sum_{a_{1:H}} \int_{r_{1:H} \in \Eps_b^c} \left(\sum_{t=1}^H |r_t| \right) \left|\Pi_{t=1}^H p_m(a_t, r_t) - \Pi_{t=1}^H g_m(a_t, r_t)\right| \pi(a_{1:H} | r_{1:H-1}) d(r_{1:H}).
\end{align*}
Note that $|r_t| < z_Z$ for all $t\in[H]$ when $r_{1:H} \in \Eps_b^c$. For each $t \in [H]$, we aim to bound 
\begin{align*}
    \sum_{a_{1:H}} \int_{r_{1:H} \in \Eps_b^c} |r_t| \left|\Pi_{t=1}^H p_m(a_t, r_t) - \Pi_{t=1}^H g_m(a_t, r_t)\right| \pi(a_{1:H} | r_{1:H-1}) d(r_{1:H}).
\end{align*}
Next, for any $r_t \in [z_s, z_{s+1}]$ for $s \in [L-1]$, we observe that
\begin{align*}
    \left|p_m(a_t, r_t) - g_m(a_t, r_t)\right| &= \frac{1}{z_{s+1}-z_s} \int_{z_s}^{z_{s+1}} \left|g_m(a_t, x) - g_m(a_t, r_t) \right| dx \\
    &\le \frac{1}{z_{s+1}-z_s} \int_{z_s}^{z_{s+1}} \left|g_m(a_t, x)\right| + \frac{1}{z_{s+1}-z_s} \int_{z_s}^{z_{s+1}} \left|g_m(a_t, x) - g_m(a_t, r_t) \right| dx.
\end{align*}
Then we observe that 
\begin{align*}
    \left|g_m(a_t, x) - g_m(a_t, r_t) \right| &= \left|\frac{d}{dx} g_m(a_t, x') (r_t - x')\right| \\
    &\le \frac{1}{\sqrt{2\pi}} |x' \exp(-x'^2/2)| |z_{s+1}-z_s|,
\end{align*}
for some $x' \in [z_s, z_{s+1}]$ where we used the mean-value theorem. A simple algebra shows that for any $x, x' \in [z_s, z_{s+1}]$, 
\begin{align*}
    |x' \exp(-x'^2/2) - x\exp(-x^2/2)| \le |r_t\exp(-r_t^2/2)| + 2|z_{s+1}-z_s|,
\end{align*}
where we used the second derivative of $x\exp(-x^2/2)$, which is $(x^2-1) \exp(-x^2/2)$, is always less than 1 in absolute value.  

Plugging above relations into bounding the difference between $p_m$ and $g_m$ yields
\begin{align*}
    \left|p_m(a_t, x) - g_m(a_t, r_t) \right| &\le \frac{1}{z_{s+1}-z_s}  \left( \frac{\epsilon^8}{H^8} \int_{z_s}^{z_{s+1}} g_m(a_t, x) dx + \int_{z_s}^{z_{s+1}} |g_m(a_t,x) - g_m(a_t,r_t)| dx \right) \\
    &\le \frac{\epsilon^8}{H^8} + \frac{1}{\sqrt{2\pi}} \int_{z_s}^{z_{s+1}} |r_t\exp(-r_t^2/2)| + 2|z_{s+1}-z_s| dx \\
    &\le \frac{3\epsilon^2}{H^4} + \frac{2\epsilon}{H^2\sqrt{2\pi}} |r_t| \cdot \exp(-r_t^2/2).
\end{align*}
Using this the above, we bound
\begin{align}
    \sum_{a_{1:H}} &\int_{r_{1:H} \in \Eps_b^c} |r_t| \left|\Pi_{t=1}^H p_m(a_t, r_t) - \Pi_{t=1}^H g_m(a_t, r_t)\right| \pi(a_{1:H} | r_{1:H-1}) d(r_{1:H}). \label{eq:target_bound}
\end{align}
If $t = H$, then 
\begin{align*}
    \eqref{eq:target_bound} &\le \sum_{a_{1:H}} \int_{r_{1:H-1}} \Pi_{t=1}^{H-1} g_m(a_t, r_t) \pi(a_{1:H}|r_{1:H-1}) d(r_{1:H-1}) \int_{r_H \in \Eps_b^c} |r_H| |p_m(a_H, r_H) - g_m(a_H, r_H)| \cdot d(r_H) \\
    &+ \sum_{a_{1:H}} \int_{r_{1:H-1} \in \Eps_b^c} \left|\Pi_{t=1}^H p_m(a_t, r_t) - \Pi_{t=1}^H g_m(a_t, r_t)\right| \pi(a_{1:H}|r_{1:H-1}) d(r_{1:H-1}) \int_{r_H \in \Eps_b^c} |r_H| p_m(a_H, r_H) \cdot d(r_H).
\end{align*}
For the first term, 
\begin{align*}
    \int_{r_H \in \Eps_b^c} |r_H| |p_m(a_H, r_H) - g_m(a_H, r_H)| \cdot d(r_H) \le z_Z \cdot (z_Z - z_1) \cdot \frac{3\epsilon^2}{H^4} + \frac{4\epsilon}{H^2} \le \frac{8\epsilon}{H^2},
\end{align*}
where we used $\int_{r_H} \frac{1}{\sqrt{2\pi}} r_H^2 \exp(-r_H^2/2) \le 1$ and $32 (\epsilon/H) \cdot \log(H/\epsilon) < 1$ for sufficiently small $\epsilon$. Note that we also have 
$$\sum_{a_{1:H}} \int_{r_{1:H-1}} \Pi_{t=1}^{H-1} g_m(a_t, r_t) \pi(a_{1:H}|r_{1:H-1}) d(r_{1:H-1}) \le 1.$$ 
For the second term, we first have
\begin{align*}
    \int_{r_H \in \Eps_b^c} |r_H| p_m(a_H, r_H) \cdot d(r_H) \le 1 + \mu_m(a_H) \le 2.
\end{align*}
Furthermore, we can show that
\begin{align*}
    \sum_{a_{1:H}} &\int_{r_{1:H-1} \in \Eps_b^c} \left|\Pi_{t=1}^H p_m(a_t, r_t) - \Pi_{t=1}^H g_m(a_t, r_t)\right| \pi(a_{1:H}|r_{1:H-1}) d(r_{1:H-1}) \\
    &= \sum_{a_{1:H-1}} \int_{r_{1:H-1} \in \Eps_b^c} \left|\Pi_{t=1}^H p_m(a_t, r_t) - \Pi_{t=1}^H g_m(a_t, r_t)\right|  \pi(a_{1:H-1}|r_{1:H-2}) \sum_{a_H} \pi(a_{H}|a_{1:H-1}, r_{1:H-1}) d(r_{1:H-1}) \\
    &= \sum_{a_{1:H-1}} \int_{r_{1:H-1} \in \Eps_b^c} \left|\Pi_{t=1}^H p_m(a_t, r_t) - \Pi_{t=1}^H g_m(a_t, r_t)\right| \pi(a_{1:H-1}|r_{1:H-2}) d(r_{1:H-1}),
\end{align*}
from which we recursively apply similar arguments. Thus we can conclude that
\begin{align*}
    \sum_{a_{1:H}} &\int_{r_{1:H} \in \Eps_b^c} |r_H| \left|\Pi_{t=1}^H p_m(a_t, r_t) - \Pi_{t=1}^H g_m(a_t, r_t)\right| \pi(a_{1:H} | r_{1:H-1}) d(r_{1:H}) \le O(\epsilon/H). 
\end{align*}
$|r_t|$ with other time steps can also be similarly bounded. Thus, we can conclude that 
\begin{align*}
    \sum_{a_{1:H}} \int_{r_{1:H}} \left(\sum_{t=1}^H |r_t| \right) \left|\Pi_{t=1}^H p_m(a_t, r_t) - \Pi_{t=1}^H g_m(a_t, r_t)\right| \pi(a_{1:H} | r_{1:H-1}) d(r_{1:H}) \le O(\epsilon),
\end{align*}
and therefore $|V(\pi) - \tilde{V}(\pi)| \le O(\epsilon)$ since $\sum_{m=1}^M w_m = 1$.

\section{Deferred Details in Section \ref{section:max_likelihood}}
\subsection{Additional Definitions}
Let us define a few notation and interaction protocol. Suppose at the beginning of episode, a latent context $m_0 \in [M]$ is chosen, and and at each time step $t \in [H]$, we play $a_{i_t}$ where $i_t$ is sampled from $\text{Unif} ([n])$. Let $\bm{i} = (i_1, i_2, ..., i_H)$ be the sequence of indices of played core actions, and $\bm{b} = (b_1, b_2, ...b_H)$ be the event-observation sequence in the episode, where $b_t := \indic{r_t = Z_{i_t}}$. Let the parameter space $\Theta$ be the set of valid parameters: 
\begin{align}
    \Theta = \{\theta = \{(w_m, \nu_m)\}_{m=1}^M| \forall j \in [n], m \in [M] \text{ s.t. } w_m, \nu_m(i) \in \mathbb{R}_+, \sum_{m=1}^M w_m = 1, \nu_m(j) \le 1\}. \label{eq:define_theta}
\end{align}
We use superscript $k$ to denote quantities observed in the $k^{\mathrm{th}}$ episode. The probability of a trajectory under a model $\theta \in \Theta$ in the $k^{\mathrm{th}}$ episode is defined by 
\begin{align*}
    \PP_{\theta}(\bm{b}^k, \bm{i}^k) := (1/n)^{H} \cdot \sum_{m=1}^M w_m \Pi_{t=1}^H (b_t^k \nu_m(i_t) + (1 - b_t^k)(1 - \nu_m(i_t))).
\end{align*}

\subsection{Polynomial Upper Bounds with Separation}\label{subsec:poly_upper_bound_sep}

In this subsection, we specify the details on separation conditions that make the polynomial sample complexity possible with MLE solutions. Suppose that there exists a context revealing action for any $m \neq m' \in [M]$, {\it i.e.,} we are given the following assumption: 
\begin{assumption}[Separated Bandit Instances]  
    \label{assumption:separation}
    For any $m \neq m' \in [M]$, there exists some (unknown) $a \in \mA$ such that $\|\mu_m^*(a, \cdot) - \mu_{m'}^*(a, \cdot)\|_1 \ge \gamma$ for some known $\gamma > 0$. 
\end{assumption}
Under Assumption \ref{assumption:separation}, if the time horizon $H = \tilde{O}(Z^2 M^2 / \gamma^2)$ is given enough to identify the context within each episode, then we can significantly improve the sample complexity for learning LMAB, from exponential to polynomial. Note that the time-horizon $H$ can be still much smaller than $A$ and thus we cannot explore all actions within a single episode, which is in contrast to explicit clustering based approaches studied in \cite{brunskill2013sample, hallak2015contextual}. 

Maximum likelihood estimator for LMABs with separation can guarantee the following:
\begin{lemma}
    \label{lemma:likelihood_separation}
    Consider the maximum likelihood estimator $\theta_N = \{(\hat{w}_m, \hat{\nu}_m)\}_{m=1}^M$ under Assumption~\ref{assumption:separation} with time-horizon $H \ge C_1 \cdot n M Z^2 \log(1/(\epsilon w_{\mathrm{min}})) / \gamma^2$ for some universal constant $C_1 > 0$. If $N = C_2 \cdot w_{\mathrm{min}}^{-2} n \cdot \log(N / \eta)/\epsilon^2$ for some large constant $C_2 > 0$, then with probability at least $1 - \eta$, we have (up to some permutations in $\theta_N$)
    \begin{align*}
        |w_m^* - \hat{w}_m| \le \epsilon w_{\mathrm{min}}, \ \|\nu_m^* - \hat{\nu}_m\|_{\infty} \le 2\epsilon, \qquad \forall m \in [M].
    \end{align*}
\end{lemma}
To get the above result, we first observe a consequence due to experimental design: the converse of Corollary \ref{corollary:action_event_sampling} implies that if Assumption \ref{assumption:separation} holds for all $m \neq m'$, then we have $\|\nu_m^* - \nu_{m'}^*\|_{\infty} \ge \gamma / (Z \sqrt{2M})$. Thus, if $H/n = \tilde{O}(M Z^2 / \gamma^2)$, then we can play each core action $O(H/n)$-times and get an $O(\gamma/(Z\sqrt{M}))$-accurate estimator $\hat{\nu}_m$ for one of $\{\nu_{m'}^*\}_{m'=1}^M$. Since we have good separation between samples from different contexts, by proper clustering arguments, the sample complexity of recovering $\theta^*$ can be polynomial. With Lemma \ref{lemma:likelihood_separation}, we can use equation \eqref{eq:value_to_wasserstein} to connect the near-optimality of returned policy computed with $\{(\hat{w}_m, \hat{\mu}_m)\}_{m=1}^M$ and the closeness in Wasserstein metric. We mention that for $H$, the dependence on $Z$ can be removed with more computationally expensive experimental design (see also Remark \ref{remark:expensive_core_set}).

\subsection{Proof of Lemma \ref{lemma:likelihood_to_moments}}
We connect the maximum likelihood estimator to total variation distance between observations from $\theta^*$ and $\theta_N$. The connection between MLE $\theta_N$ and closeness in distributions of observations $\bm{b}$ can be established by the following lemma. 
\begin{lemma}
    \label{lemma:tv_observation_bound}
    There exists a universal constant $C > 0$ such that with probability at least $1 - \eta$,
    \begin{align*}
        \sum_{\bm{b} \in \{0,1\}^H} \sum_{\bm{i} \in [n]^H} \left|\PP_{\theta_{N}} (\bm{b}, \bm{i}) - \PP_{\theta^*} (\bm{b}, \bm{i}) \right| \le C \sqrt{\frac{n\log (nHN) + \log(1/\eta)}{N}}. 
    \end{align*}
\end{lemma}
That is, total variation distance between two observation distributions is bounded by $\tilde{O}\left(\sqrt{n/N}\right)$. On the other hand, for any $l \in [\min(H, 2M-1)]$ and any multi-index $(i_1, i_2, ..., i_l) \in [n]^l$, we have
\begin{align*}
     C &\sqrt{\frac{n\log (nHN) + \log(1/\eta)}{N}} \ge \sum_{\bm{b}, \bm{i}} \left|\PP_{\theta_{N}} (\bm{b}, \bm{i}) - \PP_{\theta^*} (\bm{b}, \bm{i}) \right| \\
     &\ge \sum_{\bm{b}} \left|\PP_{\theta_{N}} - \PP_{\theta^*} \right| (b_{1:l} | a_{1:l} = (a_{i_1}, a_{i_2}, ..., a_{i_l})) \cdot \PP(a_{1:d} = (a_{i_1}, a_{i_2}, ..., a_{i_l})) \\
     &= \sum_{b_{1:l}} \left|\PP_{\theta_{N}} - \PP_{\theta^*} \right| (b_{1:l} | a_{1:l} = (a_{i_1}, a_{i_2}, ..., a_{i_l})) \cdot n^{-l} \ge n^{-l} \cdot \|\hat{T}_l - T_l\|_{\infty}.
\end{align*}
This implies that 
\begin{align*}
    \|\hat{T}_l - T_l\|_{\infty} \le C n^l \cdot \sqrt{\frac{n \log (HNn) + \log(1/\eta)}{N}}.
\end{align*}
Applying this to all $l \in [\min(H, 2M-1)]$, we get Lemma \ref{lemma:likelihood_to_moments}.

\subsection{Proof of Lemma \ref{lemma:likelihood_separation}}
From Lemma \ref{lemma:tv_observation_bound}, without loss of generality, we assume that the total variation distance between observations from $\theta_N$ and $\theta^*$ is bounded by $\epsilon w_{\mathrm{min}}/2$ since $N = \tilde{O}(w_{\mathrm{min}}^{-2} n / \epsilon^2)$:
\begin{align*}
    \sum_{\bm{b}, \bm{i}} \left|\PP_{\theta_{N}} (\bm{b}, \bm{i}) - \PP_{\theta^*} (\bm{b}, \bm{i}) \right| \le \epsilon w_{\mathrm{min}}/2.
\end{align*}
We will verify that for every $m \in [M]$, there exists $m' \in [M]$ such that $\| \hat{\nu}_m - \nu_{m'}^* \|_{\infty} \le 2\epsilon$ and $|\hat{w}_m - w_{m'}^*| \le \epsilon w_{\mathrm{min}}$. 

First note that for all $m \neq m' \in [M]$, we have $\|\nu_m^* - \nu_{m'}^*\|_{\infty} \ge \lambda := \gamma / (Z\sqrt{2M})$. If it is not, then the model does not satisfy Assumption \ref{assumption:separation}. We start with the following lemma:
\begin{lemma}
    \label{lemma:tv_event_bound_separation}
    Suppose that there exists $m \in [M]$ such that $\| \nu_m^* - \hat{\nu}_{m'} \|_{\infty} \ge \lambda / 4$ for all $m' \in [M]$. For every $j \in [n]$, define $E_{m,j}$ an event defined as:
    \begin{align*}
        E_{m,j} := \left\{ \left| \frac{\sum_{t=1}^H \indic{i_t = j} b_t}{\sum_{t=1}^H \indic{i_t = j}} - \nu_{m}^*(j) \right| < \lambda / 8 \right\},
    \end{align*}
    and let $E_m = \cap_{j=1}^n E_{m,j}$. Then
    \begin{align}
        \PP_{\theta^*}(E_m) \ge w_m / 2, \ \PP_{\theta_N} (E_m) \le \epsilon w_{\mathrm{min}}.
    \end{align}
\end{lemma}
\begin{proof}
    Let us first check that $\PP_{\theta^*}(E) \ge w_m / 2$. Let $n_j = \sum_{t=1}^H \indic{i_t = j}$. Since, 
    $$\PP_{\theta^*}(E) \ge w_m \cdot \PP_{\theta^*}(E | m_0 = m) = w_m \cdot \cap_{j=1}^n \PP_{\theta^*} (E_j | m_0 = m),$$ 
    it suffices to show that  
    \begin{align*}
        \PP_{\theta^*}(E_j | m_0 = m) &\ge \PP_{\theta^*}\left(\left| \sum_{t=1}^H \indic{i_t = j} b_t - n_j \nu_m^*(j)\right| \le n_j \lambda / 8 \ \Bigg| \ m_0 = m, n_j \ge H/(2n)\right) \PP_{\theta^*} \left(n_j \ge H/(2n) \right) \\
        &\ge \left(1 - \exp(- \lambda^2 \cdot H/(64n) )\right) \left(1 - \exp \left( \frac{-(1/2) (H/2n)^2}{H (1/n) (1-1/n) + (1/3) (H/2n)} \right) \right).
    \end{align*}
    where in the last inequality we applied Hoeffeding's concentration inequality for first term, and then used Bernstein's inequality. Plugging $H > C \cdot n M Z^2 \log (1/(\epsilon w_{\mathrm{min}})) / \gamma^2$ for some sufficiently large constant $C > 0$, we get 
    \begin{align*}
        \PP_{\theta^*}(E_j | m_0 = m) &\ge \left(1 - \exp(-2 (\gamma^2 / (128 Z^2 M)) \cdot H/(2n)) \right) \cdot \left(1 - \exp \left( -H/(16n) \right) \right) \ge 1 - (\epsilon w_{\mathrm{min}})^2.
    \end{align*}
    By union bound, we have $\PP_{\theta^*} (E) \ge w_m \cdot (1 - n \epsilon^2 w_{\mathrm{min}}^2) \ge w_m/2$.
    
    Now we check that $\PP_{\theta_N}(E) \le \epsilon w_{\mathrm{min}}$. Starting from
    \begin{align*}
        \PP_{\theta_N}(E) = \sum_{m'=1}^M w_{m'} \PP_{\theta_N} (E | m_0 = m'), 
    \end{align*}
    it suffices to show that $\PP_{\theta_N} (E | m_0 = m') \le \epsilon w_{\mathrm{min}}$ for all $m' \in [M]$. Let us fix $m'$ and define
    \begin{align*}
        E_{m',j} := \left\{ \left| \frac{\sum_{t=1}^H \indic{i_t = j} b_t}{\sum_{t=1}^H \indic{i_t = j}} - \hat{\nu}_{m}(j) \right| < \lambda / 8 \right\},
    \end{align*}
    and let $E_{m'} = \cap_{j=1}^n E_{m',j}$. Following the same argument for $\PP_{\theta^*}(E_j|m_0 = m)$, we can show that 
    \begin{align*}
        \PP_{\theta_N}(E_{m'} | m_0 = m') &\ge 1 - (\epsilon w_{\mathrm{min}})^2.
    \end{align*}
    If this happens, then it implies $E^c$ since $\|\nu_{m}^* - \hat{\nu}_{m'}\|_{\infty} \ge \lambda / 2$. Thus, 
    \begin{align*}
        \sum_{m'=1}^M \hat{w}_{m'} \PP_{\theta_N}(E_{m'} | m_0 = m') &\le \sum_{m'=1}^M \hat{w}_{m'} \PP_{\theta_N}(E^c | m_0 = m') = \PP_{\theta_N} (E^c).
    \end{align*}
    In other words, we have
    \begin{align*}
        \PP_{\theta_N} (E) \le 1 - \sum_{m'=1}^M \hat{w}_{m'} \PP_{\theta_N}(E_{m'} | m_0 = m') \le \epsilon^2 w_{\mathrm{min}}^2.
    \end{align*}
\end{proof}

The conclusion of Lemma \ref{lemma:tv_event_bound_separation} contradict that $| \PP_{\theta^*}(E_m) - \PP_{\theta_N} (E_m) | \le \epsilon w_{\mathrm{min}}/2$ due to the total variation distance bound from Lemma \ref{lemma:tv_observation_bound}. Thus, we ensure that for all $m \in [M]$, there exists $m' \in [M]$ such that $\| \nu_m^* - \hat{\nu}_{m'} \|_{\infty} \le \lambda / 4$. 

Now without loss of generality, we can ignore the permutation invariance of models and assume that 
\begin{align*}
    \|\nu_m^* - \hat{\nu}_m\|_{\infty} \le \lambda / 4, \qquad \forall m \in [M].
\end{align*}
Then we now show that for all $m \in [M]$, it holds that
\begin{align*}
    |w_m^* - \hat{w}_m^*| \le \epsilon w_{\mathrm{min}}, \quad \|\nu_m^* - \hat{\nu}_m^*\|_{\infty} \le \epsilon.
\end{align*}
For every $m$, let us define $E_m$ similarly to Lemma \ref{lemma:tv_event_bound_separation}:
\begin{align*}
    E_{m,j} := \left\{ \left| \frac{\sum_{t=1}^H \indic{i_t = j} b_t}{\sum_{t=1}^H \indic{i_t = j}} - \nu_{m}^*(j) \right| < \lambda / 2 \right\},
\end{align*}
and $E_m := \cap_{j=1}^n E_{m,j}$. Then we proceed as the following:
\begin{align*}
    \sum_{\bm{b}, \bm{i}} \left|\PP_{\theta_{N}} (\bm{b}, \bm{i}) - \PP_{\theta^*} (\bm{b}, \bm{i}) \right| &\ge \sum_{(\bm{i}, \bm{b}) \in E_m} \left|\PP_{\theta_{N}} (\bm{b}, \bm{i}) - \PP_{\theta^*} (\bm{b}, \bm{i}) \right| \\
    &= \sum_{(\bm{i}, \bm{b}) \in E_m} \left|\sum_{m'=1}^M \hat{w}_{m'}\PP_{\theta_{N}} (\bm{b}, \bm{i} | m_0 = m') - \sum_{m'=1}^M w_{m'} \PP_{\theta^*} (\bm{b}, \bm{i} | m_0 = m') \right| \\
    &\ge \sum_{\bm{i}} \sum_{\bm{b}} \left|\hat{w}_{m}\PP_{\theta_{N}} (\bm{b}, \bm{i} | m_0 = m) - w_{m} \PP_{\theta^*} (\bm{b}, \bm{i} | m_0 = m) \right| \\
    &\ \ - \left( \hat{w}_{m}\PP_{\theta_{N}} (E_m^c | m_0 = m) + w_{m} \PP_{\theta^*} (E_m^c | m_0 = m) \right) \\
    &\ \ - \sum_{m'\neq m} \left( \hat{w}_{m'}\PP_{\theta_{N}} (E_m | m_0 = m') + w_{m'} \PP_{\theta^*} (E_m | m_0 = m') \right).
\end{align*}
Recall that for any $m' \neq m$, the event $E_m$ implies $E_{m'}^c$ since every $\nu_m^*$ and $\nu_{m'}^*$ are separated by at least $\lambda = \gamma / (Z\sqrt{2M})$, which implies that $\nu_m^*$ and $\hat{\nu}_{m'}$ are separated by $(3/4) \lambda$. Following the same argument as in the proof for Lemma \ref{lemma:tv_event_bound_separation},
\begin{align*}
    \PP_{\theta_N} (E_m | m_0 = m') &\le \PP_{\theta_N} (E_{m'}^c | m_0 = m') \le \epsilon^2 w_{\mathrm{min}}^2, \\
    \PP_{\theta^*} (E_m | m_0 = m') &\le \PP_{\theta^*} (E_{m'}^c | m_0 = m') \le \epsilon^2 w_{\mathrm{min}}^2.
\end{align*}
Similarly, we can also check that
\begin{align*}
    \PP_{\theta_N} (E_m^c | m_0 = m) &\le \epsilon^2 w_{\mathrm{min}}^2, \\
    \PP_{\theta^*} (E_m^c | m_0 = m) &\le \epsilon^2 w_{\mathrm{min}}^2.
\end{align*}
Thus, we have
\begin{align*}
    \sum_{\bm{i}, \bm{b}} \left|\PP_{\theta_{N}} (\bm{b}, \bm{i}) - \PP_{\theta^*} (\bm{b}, \bm{i}) \right| \ge \sum_{\bm{b}, \bm{i}} \left|\hat{w}_{m}\PP_{\theta_{N}} (\bm{b}, \bm{i} | m_0 = m) - w_{m} \PP_{\theta^*} (\bm{b}, \bm{i} | m_0 = m) \right| - 2 \epsilon^2 w_{\mathrm{min}}^2. 
\end{align*}
We remain to lower bound $\sum_{\bm{b}, \bm{i}} \left|\hat{w}_{m}\PP_{\theta_{N}} (\bm{b}, \bm{i} | m_0 = m) - w_{m}^* \PP_{\theta^*} (\bm{b}, \bm{i} | m_0 = m) \right|$. First note that
\begin{align*}
    \epsilon w_{\mathrm{min}} / 2 \ge \sum_{\bm{b}, \bm{i}} \left|\hat{w}_{m}\PP_{\theta_{N}} (\bm{b}, \bm{i} | m_0 = m) - w_{m}^* \PP_{\theta^*} (\bm{b}, \bm{i} | m_0 = m) \right| \ge |\hat{w}_m - w_m^*|.
\end{align*}
Thus we have $|\hat{w}_m - w_m^*| \le \epsilon w_{\mathrm{min}}$. Now given this, we can proceed as 
\begin{align*}
    \epsilon w_{\mathrm{min}} /2 &\ge \sum_{\bm{b}, \bm{i}} \left|\hat{w}_{m}\PP_{\theta_{N}} (\bm{b}, \bm{i} | m_0 = m) - w_{m}^* \PP_{\theta^*} (\bm{b}, \bm{i} | m_0 = m) \right| \\
    &\ge w_m^* \cdot \sum_{\bm{b}, \bm{i}} \left|\PP_{\theta_{N}} (\bm{b}, \bm{i} | m_0 = m) - \PP_{\theta^*} (\bm{b}, \bm{i} | m_0 = m) \right| - |\hat{w}_m - w_m^*|.
\end{align*}
Now let $i_m := arg\max_{j \in [n]} |\nu_m^* (j) - \hat{\nu}_m(j)|$. Define an event $b_t$ being $1$ at the first time $a_{i_m}$ is played:
\begin{align*}
    F := \{b_t = 1, t = arg\min_{t' \in [H]} i_{t'} = i_m\}.
\end{align*}
Note that since $H \gg n\log (1/(\epsilon w_{\mathrm{min}}))$, $i_m$ is played at least once with probability at least $1 - (\epsilon w_{\mathrm{min}})^2$. Then we can lower bound the total variation distance as 
\begin{align*}
    \sum_{\bm{b}, \bm{i}} \left|\PP_{\theta_{N}} (\bm{b}, \bm{i} | m_0 = m) - \PP_{\theta^*} (\bm{b}, \bm{i} | m_0 = m) \right| &\ge |\PP_{\theta_{N}} (F | m_0 = m) - \PP_{\theta^*} (F | m_0 = m)| \\
    &\ge |\nu_m^* (i_m) - \hat{\nu}_m (i_m)| - 2 (\epsilon w_{\mathrm{min}})^2.
\end{align*}
Hence we can conclude that
\begin{align*}
    \epsilon w_{\mathrm{min}}/2 \ge w_m |\nu_m^* (i_m) - \hat{\nu}_m (i_m)| - 2 (\epsilon w_{\mathrm{min}})^2 - \epsilon w_{\mathrm{min}},
\end{align*}
which implies $\|\nu_m^* - \hat{\nu}_m\|_{\infty} \le 2\epsilon$. 

Note that the Wasserstein distance can thus be bounded as
\begin{align*}
    W(\gamma^*, \hat{\gamma}) \le \sum_{m=1}^M |w_m^* - \hat{w}_m| + \sum_{m=1}^M w_m^* \|\nu_m^* - \hat{\nu}_m\|_{\infty} \le \epsilon M w_{\mathrm{min}} + 2\epsilon \le 3\epsilon.
\end{align*}
\begin{remark}
    In our guarantee, we assumed that $H \gg \log(1/\epsilon)$, {\it i.e.,} $H$ should be increased logarithmically with the final accuracy. As one might imagine, this is not the optimal condition for separations between individual models. A more delicate and technically involved analysis might reveal that the sub-optimal dependency on $\log(1/\epsilon)$ can be dropped with the EM algorithm as in learning Gaussian mixture models ({\it e.g.,} \cite{kwon2020algorithm}). Since it is technically much more complicated, we leave the task of verifying more tight separation conditions as future work.  
\end{remark}

\subsection{Proof of Lemma \ref{lemma:tv_observation_bound}}
This is a rather standard consequence of MLE for parameterized distributions. Let us define well-conditioned parameters $\Theta' \subseteq \Theta$ defined as follows: 
\begin{align*}
    \Theta' = \Big\{\theta = \{(w_m, \nu_m)\}_{m=1}^M| &\forall j \in [n], m \in [M] \text{ s.t. } w_m, \ \nu_m(j) \in \mathbb{R}_+, \\
    &\sum_{m=1}^M w_m = 1, \\
    &\epsilon \le w_m, \ \epsilon \le \nu_m(j) \le 1 - \epsilon, \quad \forall m \in [M], j \in [n] \Big\}. \label{eq:define_theta}
\end{align*}
Let $\Theta_\epsilon$ be $\epsilon^2$-covering of $\Theta'$. Since there are $n + M$ free parameters, log-cardinality of $\Theta_\epsilon$ is at most $O(1/\epsilon)^{2(n+M)}$. Note that $\Theta_\epsilon$ is a (not necessarily minimal) $\epsilon$-cover for $\Theta$ as well.

To simplify the notation, we often use $X := (\bm{b}, \bm{i})$ (and $X^k = (\bm{b}^k, \bm{i}^k)$) to replace a sample trajectory. Our goal is to bound
\begin{align*}
    TV(\theta_N, \theta^*) := \sum_{X = (\bm{b}, \bm{i}): \bm{b} \in \{0,1\}^H, \bm{i} \in [n]^H} &|\PP_{\theta_N} (X) - \PP_{\theta^*} (X)|,
\end{align*}
where $TV(\theta_1, \theta_2)$ is a total variation distance between $\PP_{\theta_1}$ and $\PP_{\theta_2}$ for any $\theta_1, \theta_2 \in \Theta$.

Let $\overline{\theta} = \{(\overline{w}_m, \overline{\nu}_m\}_{m=1}^M \in \Theta_\epsilon$ such that
\begin{align*}
    \frac{w_m^*}{\overline{w}_m} &\le 1 + 2 \epsilon, & \forall m \in [M], \\
    \frac{\nu_m^*(i)}{\overline{\nu}_m(i)} &\le 1 + 2 \epsilon, & \forall m \in [M], i \in [n], \\
    \frac{1 - \nu_m^*(i)}{1 - \overline{\nu}_m(i)} &\le 1 + 2 \epsilon, & \forall m \in [M], i \in [n].
\end{align*}
Such $\overline{\theta}$ is guaranteed to exist in $\Theta_\epsilon$ by construction. A simple algebra shows that for any trajectory $X$, we have
\begin{align*}
    \frac{\PP_{\theta^*}(X)}{\PP_{\overline{\theta}}(X)} \le (1+2\epsilon)^H.
\end{align*}
As long as $\epsilon < 1/H^2$, this is bounded by constant.

Now, fix any $\theta \in \Theta_\epsilon$ and let $l(X) := \frac{1}{2} \log \left( \frac{\PP_{\theta} (X)}{\PP_{\overline{\theta}} (X)} \right)$. Then using Chernoff's method, we get
\begin{align*}
    \PP_{\theta^*} &\left( \sum_{k=1}^N l(X^k) - \log\left( \Exs_{\theta^*} \left[ \exp\left(\sum_{k=1}^N l(X^k) \right) \right] \right) > \lambda \right) \\
    &= \PP_{\theta^*} \left( \exp\left( \sum_{k=1}^N l(X^k) - \log \left( \Exs_{\theta^*} \left[ \sum_{k=1}^N \exp(l(X^k)) \right] \right)\right) > \exp(\lambda) \right) \\
    &\le \Exs_{\theta^*} \left[\exp \left(\sum_{k=1}^N l(X^k) - \log\left( \Exs_{\theta^*} \left[ \exp \left(\sum_{k=1}^N l(X^k) \right) \right] \right) \right)\right] \cdot \exp(-\lambda) \\
    &= \exp(-\lambda),
\end{align*}
where we used $\Exs[\exp(S - \log( \Exs[ \exp(S) ] ))] = \frac{\Exs[\exp(S)]}{\Exs[ \exp(S) ]} = 1$ for any $S$ and distribution. Taking union bound over $\Theta_\epsilon$, we can conclude that with probability at least $1 - \eta$, 
\begin{align*}
    - \log\left( \Exs_{\theta^*} \left[ \Pi_{k=1}^N \sqrt{\frac{\PP_{\theta} (X^k)}{\PP_{\overline{\theta}} (X^k)}} \right] \right) &\le -\frac{1}{2} \sum_{k=1}^N \log \left( \frac{\PP_{\theta} (X^k)}{\PP_{\overline{\theta}} (X^k)} \right) + \log(|\Theta_\epsilon| / \eta).
\end{align*}
Now, since $X^1, ..., X^N$ are independent, we have
\begin{align*}
    - \log\left( \Exs_{\theta^*} \left[ \Pi_{k=1}^N \sqrt{\frac{\PP_{\theta} (X^k)}{\PP_{\overline{\theta}} (X^k)}} \right] \right) &= - \sum_{k=1}^N \log\left( \Exs_{\theta^*} \left[ \sqrt{\frac{\PP_{\theta} (X)}{\PP_{\overline{\theta}} (X)}} \right] \right) \\
    &= - \sum_{k=1}^N \log\left( \Exs_{\theta^*} \left[ \sqrt{\frac{\PP_{\theta} (X)}{\PP_{\theta^*} (X)}} \cdot \sqrt{\frac{\PP_{\theta^*} (X)}{\PP_{\overline{\theta}} (X)}} \right] \right) \\
    &\ge - \sum_{k=1}^N \log\left( (1+2\epsilon)^{H/2} \cdot \Exs_{\theta^*} \left[ \sqrt{\frac{\PP_{\theta} (X)}{\PP_{\theta^*} (X)}} \right] \right) \\
    &= - \frac{HN}{2} \log(1+2\epsilon) - \sum_{k=1}^N \log \left( \Exs_{\theta^*} \left[ \sqrt{\frac{\PP_{\theta} (X)}{\PP_{\theta^*} (X)}} \right] \right). 
\end{align*}
Using $-\ln(x) \ge 1 - x$, we get
\begin{align*}
    - \sum_{k=1}^N \log\left( \Exs_{\theta^*} \left[ \sqrt{\frac{\PP_{\theta} (X)}{\PP_{\theta^*} (X)}} \right] \right) &\ge \sum_{k=1}^N \left( 1 - \Exs_{\theta^*} \left[ \sqrt{\frac{\PP_{\theta} (X)}{\PP_{\theta^*} (X)}} \right] \right) \\
    &= \sum_{k=1}^N \left( 1 - \sum_{X} \sqrt{\PP_{\theta^*} (X) \PP_{\theta} (X)} \right) = N \mathcal{H}^2(\theta^*, \theta),
\end{align*}
where $\mathcal{H} (\theta_1, \theta_2)$ is a Hellinger distance between $\PP_{\theta_1}$ and $\PP_{\theta_2}$. Also note that $TV(\theta^*, \theta) \le \mathcal{H} (\theta^*, \theta)$. Collecting all, we now can say that
\begin{align*}
    N \cdot TV^2 (\theta^*, \theta) \le -\frac{1}{2} \sum_{k=1}^N \log\left( \frac{\PP_{\theta}(X^k)}{\PP_{\overline{\theta}} (X^k)} \right) + \log(\Theta_\epsilon / \eta) + \frac{HN}{2} \log(1+2\epsilon),
\end{align*}
with probability at least $1 - \eta$. Finally, let $\overline{\theta}_N \in \Theta_\epsilon$ be the one close to $\theta_N$ such that
\begin{align*}
    \frac{\PP_{\theta_N} (X)}{\PP_{\overline{\theta}_N}(X)} \le (1+2\epsilon)^H,
\end{align*}
similarly to when defining $\overline{\theta}$. Then, 
\begin{align*}
    N \cdot TV^2 (\theta^*, \overline{\theta}_N) &\le -\frac{1}{2} \sum_{k=1}^N \log\left( \frac{\PP_{\theta_N}(X^k)}{\PP_{\overline{\theta}} (X^k)} \right) +\frac{1}{2} \sum_{k=1}^N \log\left( \frac{\PP_{\theta_N}(X^k)}{\PP_{\overline{\theta}_N} (X^k)} \right) + \log(\Theta_\epsilon / \eta) + \frac{HN}{2} \log(1+2\epsilon) \\
    &\le \log(\Theta_\epsilon / \eta) + HN \log(1+2\epsilon),
\end{align*}
where we used the fact that $\theta_N$ is the maximum likelihood estimator. With a proper scaling of $\epsilon \ll 1/(nHN)^4$, we get
\begin{align*}
    TV(\theta^*, \overline{\theta}_N) &\le O \left(\sqrt{\frac{(n+M) \log(nHN) + \log(1/\eta) }{N}} \right).
\end{align*}

Finally, it is not hard to show that $TV(\theta_N, \overline{\theta}_N) \le 2 \epsilon H \le 1/N$. We can conclude that 
\begin{align*}
    TV(\theta^*, \theta_N) \le C \cdot \sqrt{\frac{n \log(nHN) + \log(1/\eta)}{N}},
\end{align*}
for some sufficiently large constant $C > 0$.

\end{appendices}

\end{document}